\documentclass[pdflatex,sn-mathphys-num]{sn-jnl}

\usepackage{graphicx}%
\usepackage{multirow}%
\usepackage{amsmath,amssymb,amsfonts}%
\usepackage{amsthm}%
\usepackage{mathrsfs}%
\usepackage[title]{appendix}%
\usepackage{xcolor}%
\usepackage{textcomp}%
\usepackage{manyfoot}%
\usepackage{booktabs}%
\usepackage{algorithm}%
\usepackage{algorithmicx}%
\usepackage{algpseudocode}%
\usepackage{listings}%
\usepackage{comment}
\usepackage{braket}
\usepackage{bm}
\usepackage{comment}
\usepackage{subcaption}

\usepackage{thmtools}
\usepackage{thm-restate}

\theoremstyle{thmstyleone}%
\newtheorem{theorem}{Theorem}
\newtheorem{corollary}[theorem]{Corollary}

\theoremstyle{thmstyletwo}%

\theoremstyle{thmstylethree}%
\newtheorem{definition}{Definition}%

\newtheorem{lemma}{Lemma}[section]

\newcommand{\norm}[1]{\left\lVert#1\right\rVert}
\newcommand{\abs}[1]{\left\vert#1\right\vert}
\newcommand{\E}{\mathbb{E}}
\newcommand{\R}{\mathbb{R}}

\newcommand{\bigOmega}{\Omega}
\newcommand{\xbf}{\mathbf{x}}
\newcommand{\Kbf}{\mathbf{K}}
\newcommand{\Ibf}{\mathbf{I}}
\newcommand{\hbf}{\mathbf{h}}
\newcommand{\phibf}{\bm{\phi}}
\newcommand{\Zbf}{\mathbf{Z}}
\newcommand{\Jbf}{\mathbf{J}}
\newcommand{\bbf}{\mathbf{b}}
\newcommand{\thetabf}{\bm{\theta}}

\newcommand{\zbf}{\mathbf{z}}
\newcommand{\gbf}{\mathbf{g}}

\begin{document}

\title[]{Quantum-Enhanced Neural Contextual Bandit Algorithms}


\author*[1]{\fnm{Yuqi} \sur{Huang}}\email{e0727232@u.nus.edu}

\author[1,2]{\fnm{Vincent Y. F} \sur{Tan}}\email{vtan@nus.edu.sg}

\author[3]{\fnm{Sharu Theresa} \sur{Jose}}\email{s.t.jose@bham.ac.uk}

\affil*[1]{\orgdiv{Department of Mathematics}, \orgname{National University of Singapore}, \orgaddress{\city{Singapore}, \postcode{119077}, \country{Singapore}}}

\affil[2]{\orgdiv{Department of Electrical and Computer Engineering}, \orgname{National University of Singapore}, \orgaddress{\city{Singapore}, \postcode{117583},  \country{Singapore}}}

\affil[3]{\orgdiv{School of Computer Science}, \orgname{University of Birmingham}, \orgaddress{ \city{Birmingham}, \postcode{B15 2TT}, \country{United Kingdom}}}


\abstract{Stochastic contextual bandits  are fundamental for sequential decision-making but pose significant challenges for existing neural network-based algorithms, particularly when scaling to quantum neural networks (QNNs) due to issues such as massive over-parameterization, computational instability, and the barren plateau phenomenon. This paper introduces the \emph{Quantum Neural Tangent Kernel-Upper Confidence Bound} (QNTK-UCB) algorithm, a novel algorithm that leverages the Quantum Neural Tangent Kernel (QNTK) to address these limitations.

By freezing the QNN at a random initialization and utilizing its static QNTK as a kernel for ridge regression, QNTK-UCB bypasses the unstable training dynamics inherent in explicit parameterized quantum circuit training while fully exploiting the unique quantum inductive bias. For a time horizon $T$ and $K$ actions, our theoretical analysis reveals a significantly improved parameter scaling of $\Omega((TK)^3)$ for QNTK-UCB, a substantial reduction compared to $\Omega((TK)^8)$  required by classical NeuralUCB algorithms for similar regret guarantees. Empirical evaluations on non-linear synthetic benchmarks and quantum-native variational quantum eigensolver  tasks demonstrate QNTK-UCB's superior sample efficiency in low-data regimes. This work highlights how the inherent properties of QNTK provide implicit regularization and a sharper spectral decay, paving the way for achieving ``quantum advantage'' in online learning.}




\maketitle

\section{Introduction}\label{sec1}
Stochastic contextual bandits (SCBs) have been extensively studied over the past few decades due to their wide-ranging applications in areas such as clinical trials \cite{tewari2017ads,varatharajah2022contextual}, e-commerce recommendation systems \cite{li2010contextual,tang2014ensemble}, online advertising, and personalized media delivery \cite{bouneffouf2012contextual}. SCBs provide a canonical framework for online sequential decision-making, in which a learner repeatedly selects actions based on observed contextual information. At each time step, each action (e.g., trial drug) comes with contextual features that depend on side-information about the environment (e.g., age, sex, tumor biomarkers of the patient).  The learner observes these context-dependent features, selects an action, and receives a stochastic reward (e.g., observed treatment outcome). The objective of the learner is to design a decision policy that maximizes the expected cumulative reward over a finite time horizon $T$, or equivalently, minimizes the cumulative regret relative to a baseline policy. Achieving this objective requires effectively balancing exploration of alternative actions with exploitation of the currently best-performing actions.

In the classical literature, SCBs are often analyzed via linear reward models, where the unknown mean reward function is assumed to be a linear function of the context feature vector. This has led to the development of several powerful algorithms such as linear contextual UCB~\cite{linucb2011,chu2011contextual} and linear contextual Thompson sampling~\cite{agrawal2013thompson, oh2019thompson}. While linear reward models are theoretically convenient, they  often fail to capture the highly non-linear dependencies encountered in practical scenarios. This limitation has led to the exploration of several non-linear models, including generalized linear models~\cite{filippi2010parametric, li2017provably}, kernel models~\cite{valko2013finite}, and Gaussian processes~\cite{Srinivas_2012}, where the reward function is assumed to reside within a reproducing kernel Hilbert space (RKHS). Although these methods are powerful, their effectiveness largely depends on the compatibility of their inductive bias with the true underlying reward function. 

To address this challenge, classical neural networks (NNs) have been recently introduced to the bandit setting \cite{riquelme2018deep, zahavy2019deep, neuralucb, zhang2020neural}, leveraging their immense representational power to approximate complex mean reward functions. The resulting neural contextual bandit algorithms typically operate in the ``Neural Tangent Kernel" (NTK) regime, where the NN is sufficiently over-parameterized such that its training dynamics are linearized. Despite the empirical success, classical neural bandits face significant challenges: they require massive over-parameterization (with the number of parameters scaling as $\bigOmega((TK)^8)$, 
where $K$ denotes the number of actions) to satisfy convergence guarantees, and their computational cost in the online setting remains a critical bottleneck due to the need for frequent inversion of a dynamic design matrix.

Recently, quantum neural networks (QNNs), or parameterized quantum circuits (PQCs), have emerged as a powerful machine learning paradigm. Leveraging the principles of superposition and entanglement, QNNs   offer representational advantages over classical NNs with a comparable number of parameters
\cite{schuld2020circuit,du2020expressive, abbas2021power}. On the one hand, QNNs can embed classical contextual data into an exponentially large feature Hilbert space via complex quantum feature maps, offering a distinct ``quantum inductive bias" that may represent the reward functions more efficiently \cite{schuld2021effect}.  On the other hand,   recent works suggest that for data arising from inherently quantum physical processes---such as in finding the ground state of a Hamiltonian or classifying quantum phases---classical models may be inefficient \cite{kubler2021inductive, uvarov2020machine}. 

Motivated by these advantages, this work proposes a new class of quantum-enhanced neural contextual bandit algorithms that leverage QNN-based reward models. However, transitioning from classical to quantum neural bandits presents significant technical challenges. First, deep QNNs are plagued by the ``barren plateau" phenomenon \cite{mcclean2018barren}, where network gradients vanish exponentially with the number of qubits $m$. This necessitates an exponential number of measurements to estimate gradients accurately, thereby nullifying potential quantum advantages and rendering gradient-descent training unstable. Second, online training of QNN-based bandits requires computing dynamic gradient feature maps that evolve at each iteration as weights are updated. This forces the re-calculation and inversion of the design matrix at every step, leading to prohibitive computational costs for NISQ hardware.


A potential approach to overcoming barren plateaus is to restrict the architecture to shallow-depth QNNs, where the number of layers scales at most logarithmically with the number of qubits \cite{napp2022quantifying}. Interestingly, in such barren plateau-free regimes, recent results show that the training dynamics of the QNN are governed by a fixed analytic kernel determined by the circuit architecture at initialization \cite{abedi2023quantum,Girardi2025}, known as the Quantum Neural Tangent Kernel (QNTK). This represents an extension of classical NTK theory to over-parameterized QNNs. Importantly, the QNTK framework is applicable to  architectures whose depths can scale with the number of qubits, facilitating the use of circuits that are classically hard to simulate and thus yielding potential quantum advantage (see \cite[Section 2.5]{Girardi2025} for examples of such circuits). 

Motivated by these results and the challenges of explicitly training deep QNNs in a bandit setting, we propose a kernelized approach leveraging the QNTK as a static kernel for ridge regression. This strategy allows us to circumvent the non-convex optimization landscape of variational circuits while retaining the unique inductive bias inherent to the quantum feature map. Crucially, we demonstrate that the quantum feature space allows for more efficient linearization than its classical counterpart.

Our primary contributions are as follows:  
\begin{itemize}
\item  We introduce QNTK-UCB, the first contextual bandit algorithm that utilizes the empirical QNTK for reward estimation. This framework allows the learner to exploit quantum expressive power without the instabilities of explicit PQC training.
    \item We provide a comprehensive regret analysis of QNTK-UCB in terms of quantum effective dimension. A key finding of our work is that QNTK-UCB requires  significantly lower parameter scaling, $\tilde{\Omega}((TK)^3)$, to achieve the same regret bounds that require $\tilde{\Omega}((TK)^8)$ parameters in classical NeuralUCB \cite{neuralucb}.
    \item Through a series of experiments on non-linear synthetic benchmarks and quantum initial state recommendation for Variational Quantum Eigensolver (VQE) tasks, we demonstrate that QNTK-UCB exhibits superior sample efficiency in low-data regimes, providing a clear path toward ``quantum advantage" in online learning.
\end{itemize}

\textbf{Related Works:} Our work sits at the intersection of quantum machine learning and online decision-making, departing from several established lines of research in the field of quantum bandits.

A significant body of literature \cite{wan2023quantum, dai2023quantum, hikima2024quantum, siam2025quantum} proposes quantum algorithms for classical multi-armed and contextual bandits. 
These works typically assume the existence of a quantum reward oracle to achieve (quadratic) speedups in query complexity by resorting to quantum algorithms such as quantum Monte Carlo or amplitude amplification. However, the practical utility of these approaches is often hindered by the ``input bottleneck" phenomenon, namely that the computational cost of encoding classical reward data into a quantum oracle can be prohibitive, potentially neutralizing any algorithmic speedup. In contrast, our work does not assume a quantum oracle; instead, we utilize quantum circuits as a function approximation tool for classical data.

Another line of works \cite{lumbreras2022multi, brahmachari2024quantum} formulates the learning of quantum state properties, such as shadow tomography or Hamiltonian estimation, as a stochastic quantum bandit problem. However, these often fall under classical linear contextual frameworks. While QNTK-UCB is capable of addressing such tasks, it is designed as a general-purpose learner for both classical and quantum-native reward functions. 

\section{Background Setup and Quantum Neural Networks}\label{sec2}

\subsection{Problem Setting: Contextual Bandits}
We consider the stochastic $K$-armed contextual bandit problem with a finite horizon $T \in \mathbb{N}$. 
At each round $t \in [T] := \{1, \dots, T\}$, the agent observes a set of {\em context vectors} $\mathcal{X}_t = \{\xbf_{t,a}: a \in [K]\}$, where $\xbf_{t,a} \in \mathcal X := \cup_{t\in \mathbb{N}}\mathcal{X}_t \subset \mathbb R^d$ denotes the $d$-dimensional feature vector associated with arm $a$.
The agent selects an arm $a_t \in [K]$ and observes a noisy scalar reward $r_{t, a_t}$. We assume that the reward is generated as
\begin{align}
    r_{t,a_t}=h(\xbf_{t,a_t})+ \xi_t, \label{eq:reward}
\end{align} where $h: \mathbb{R}^d \to [0, 1]$ is an \textit{unknown, $[0,1]$-bounded, mean reward function}  and   $\xi_t$ is $\nu$-sub-Gaussian noise conditioned on the history  $\mathcal{H}_{t-1} = \{(\xbf_{s, a_s}, r_{s, a_s})\}_{s=1}^{t-1}$  up to and including time $t-1$, i.e., it  satisfies $\mathbb{E}[\xi_t| \mathcal{H}_{t-1}]=0$ and $\mathbb{E}[\exp(s\xi_t)|\mathcal{H}_{t-1}]\le \exp(\nu^2s^2/2)$ for all $s\in\R$. We note that the assumption of the mean reward function $h(\cdot)$ being bounded is standard in the bandit literature, and is satisfied under standard boundedness assumptions on contexts and model class (e.g., bounded contexts/parameters for linear models and bounded RKHS norm for kernel models). 
For the purpose of kernel analysis, we denote the collection of all contexts across the horizon as a ``vectorized'' dataset $\mathcal{X}_{1:TK} := \{\xbf_{t,a}\}_{t \in [T], a \in [K]}$, which may be indexed as $\{\xbf^i\}_{i=1}^{TK}$.

Let $a_t^* \in \arg\max_{a \in [K]} h(\xbf_{t,a})$ denote any optimal arm that maximizes the mean reward at round $t$. The goal of the agent is to minimize the \textit{expected cumulative regret}, 
\begin{align}\bar{R}_T:= \sum_{t=1}^T \mathbb{E}\left[ h(\xbf_{t, a_t^*}) - h(\xbf_{t, a_t}) \right], \label{eq:cum_regret} \end{align}
defined as the cumulative difference between the optimal expected reward and the expected reward of the selected arm accumulated over the horizon $T$.

\subsection{Quantum Neural Networks} \label{sec:QNN}
In this section, we explain the structure of QNNs under consideration.  Specifically, we follow the general QNN framework considered in \cite{Girardi2025}, which guarantees convergence (in distribution) of the function described by QNN to a Gaussian process in the infinite-{\em width} (number of qubits) limit. 

The QNN acts on a system of $m$ qubits with circuit depth $L \in\mathbb{N}$. In particular, we allow the number of layers $L(m)$ to vary with $m$.
The total unitary operation $U(\bm{\theta}, x)$ acting on the initial state $|0\rangle^{\otimes m}$ is composed of a sequence of $L$ layers: \begin{align}U(\boldsymbol{\theta}, \xbf) = U_L(\boldsymbol{\theta}_L, \xbf) \dots U_1(\boldsymbol{\theta}_1, \xbf).\label{eq:unitary_perlayer}\end{align}
Each layer $l \in [L]$ is represented by a unitary consisting of a parameterized block and a fixed block:
$$U_l(\boldsymbol{\theta}_l, \xbf) = W_l(\boldsymbol{\theta}_l) V_l(\xbf),$$ where $W_l(\boldsymbol{\theta}_l)$ contains trainable single-qubit rotations (parameterized by $\boldsymbol{\theta}_l$) acting on each qubit, and $V_l(\xbf)$ consists of fixed entangling gates (e.g., CNOTs) and, optionally, data-encoding gates. 
For QNNs with fixed number of layers $L$ (i.e., $L$ does not vary with $m$), it can be easily seen that $\thetabf \in \mathbb{R}^p$, with total number of parameters $p \approx Lm$.

The quantum circuit described above defines a reward model $f(\xbf;\thetabf)$ as follows: The total unitary operation $U(\thetabf,\xbf)$ acts on an initial quantum state $\vert 0^m \rangle$ to yield an output quantum state $\vert \psi(\thetabf,\xbf)\rangle= U(\thetabf,\xbf)\vert 0^m \rangle$. This output state is then measured using an observable $\mathcal{O}$. Following the framework outlined in \cite{Girardi2025}, we define $\mathcal{O}$ as a sum of local, single-qubit observables, expressed as:
\[
\mathcal{O} = \sum_{k=1}^m \mathcal{O}_k, \quad \text{where } \text{Tr}(\mathcal{O}_k) = 0,  \hspace{0.1cm}\mbox{for} \hspace{0.1cm} k=1,\hdots,m.
\]
Furthermore, the eigenspectrum of each observable $\mathcal{O}_k$ is confined to the set $\{-1, +1\}$, meaning that each eigenvalue of  $\mathcal{O}_k$  can be either be $+1$ or $-1$.
The output model $f(\xbf; {\thetabf})$ is  then  defined as the expected value of the global observable $\mathcal{O}$ with respect to the output state $\vert \psi(\thetabf,\xbf)\rangle$ up to a normalization constant $N(m)$ as
\begin{equation}\label{eq:qnn-model}
f(\xbf; {\thetabf}) = \frac{1}{N(m)} \sum_{k=1}^m f_k( \xbf;{\thetabf}), \quad \text{where} \quad f_k(\xbf;{\thetabf}) = \langle 0^m | U^\dagger({\thetabf}, \xbf) \mathcal{O}_k U({\thetabf}, \xbf) | 0^m \rangle.
\end{equation} 
Here, $N(m)$ is a normalization factor determined by the covariance function of the QNN model at initialization, i.e., when the parameters $\thetabf$ are randomly chosen at the start  before training. This normalization is included in the model definition to ensure that $f(\xbf;\thetabf)$ converges to a non-trivial Gaussian process as $m \to \infty$. Concretely, we make the following assumption \cite{Girardi2025}:

\begin{restatable}{assumption}{fcvgtogp} \label{fcvgtogp}
The distribution of parameters $\bm{\theta}$, the architecture of the quantum circuit and the
normalization $N(m)$ chosen are such that $$
\mathbb E[f_k(\xbf;\bm{\theta})] = 0  \quad\forall\, \xbf \in \mathcal{X},
$$ 
and
$$
\lim_{m\rightarrow \infty} \sup_{\xbf,\xbf' \in \mathcal X} \left| \mathbb E[f_k(\xbf;\bm{\theta})f_k(\xbf';\bm{\theta})] - \mathcal K(\xbf,\xbf') \right| = 0,
$$
where $\mathcal K:\mathcal{X} \times \mathcal{X} \rightarrow \mathbb R$ is an arbitrary bivariate function from the feature space to the real
numbers with strictly positive diagonal elements, i.e., $\mathcal K(\xbf,\xbf)>0$ for all  $\xbf \in \mathcal X$. 
\end{restatable} 

 From the above assumption, $N(m)$ is chosen so that the limit $m \rightarrow \infty$ yields a finite nontrivial covariance
function. Note that the value of $N(m)$ depends on the specific QNN architecture. For instance, the QNN architecture in Fig.~\ref{fig:circuit} has  $N(m)=\sqrt{m}$ \cite{abedi2023quantum}.
\begin{figure*}[t!] 
  \centering 
    \includegraphics[width=0.9\linewidth]{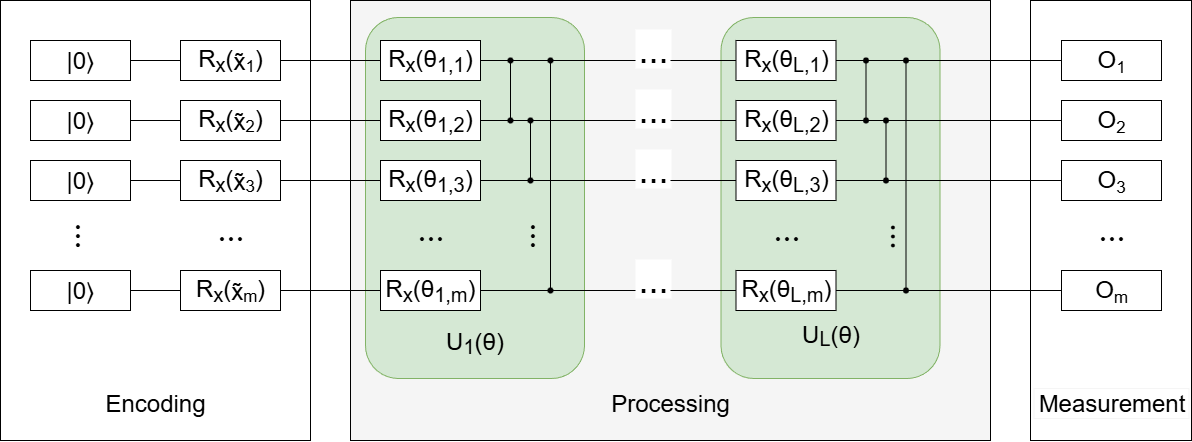}
    \caption{An example of a circuit structure}
  \label{fig:circuit}
\end{figure*}

In this work, we model the unknown reward function $h(\xbf)$ via the function $f(\xbf; {\thetabf})$, defined using QNN architectures satisfying Assumption~\ref{fcvgtogp}. We denote the gradient of the model function with respect to the parameter vector as $\gbf(\xbf; {\thetabf}) = \nabla_{\thetabf} f(\xbf; {\thetabf})$.

\subsection{Quantum Neural Tangent Kernel Theory}\label{sec:qntk}
The theoretical analysis of classical neural bandit algorithms, such as NeuralUCB, is built upon the neural tangent kernel (NTK) \cite{neuralucb, ntk} theory. This theory posits that in the ``lazy training" regime of infinitely wide neural networks, the network weights stay close to their initialization, and the optimization dynamics can be approximated by a linear model of  the network gradients. In this regime, the NN effectively operates as a linear model of the high-dimensional feature map, $\phibf(\xbf) = \gbf(\xbf; {\thetabf}_0)$, defined by the network's gradient at initialization $\thetabf_0$. 

The NTK theory has been recently extended to QNN models \cite{abedi2023quantum, Girardi2025}. In particular, Girardi and De Palma~\citep{Girardi2025} show that   under certain assumptions,  QNN functions converge to Gaussian processes in the limit as $m \rightarrow \infty$, and their dynamics are governed by the quantum neural tangent kernel (QNTK). We define the \emph{empirical QNTK} at a random initialization ${\thetabf}_0$ as:\begin{equation}\label{eq:qntk-empirical}\hat{\Kbf}_{{\thetabf}_0}(\xbf, \xbf') = \frac{1}{N_K(m)} \gbf(\xbf; {\thetabf}_0)^\top \gbf(\xbf'; {\thetabf}_0)= \frac{1}{N_K(m)} \langle \nabla_{\bm{\theta}} f(\xbf; {\thetabf}_0), \nabla_{\bm{\theta}} f(\xbf'; {\thetabf}_0) \rangle,\end{equation} 
where $\gbf(\xbf; {\thetabf}_0) = \nabla_{\thetabf} f(\xbf; {\thetabf}_0)$ is the gradient of the QNN model \eqref{eq:qnn-model} at initialization, and $N_K(m)$ is a width-dependent normalization factor chosen to ensure the kernel converges to a non-trivial limit as $m \to \infty$. 

The associated \emph{analytic QNTK} is the expectation of the empirical kernel over the random initialization of network parameters, i.e.,  \begin{align}\Kbf^{(m)}(\xbf, \xbf') = \mathbb{E}_{{\thetabf}_0}[\hat{\Kbf}_{{\thetabf}_0}(\xbf, \xbf')], \end{align} where the elements of the random vector $\bm{\theta}_0$ are  independent and uniform random variables in $[0,\pi]$. Note that the analytic QNTK  $\Kbf^{(m)}(\xbf, \xbf')$ depends on the QNN architecture. Its spectral properties, and thus the learning inductive bias, are governed by the number of qubits $m$, the circuit depth $L$, the connectivity of entangling gates, and the data encoding strategy $V(\xbf)$.  For brevity in what follows, we remove the dependence of $\Kbf^{(m)}(\xbf, \xbf')$ on $m$ and write it as $\Kbf(\xbf, \xbf')$. Similarly, when it is clear from the context, we remove the dependence of $\hat{\Kbf}_{\thetabf_0}$ on $\thetabf_0$ and write it as $\hat{\Kbf}$.
Furthermore, we make the following assumption \cite{Girardi2025}:

\begin{restatable}{assumption}{cvgtolimqntk}\label{ass:cvgtolimqntk}
    There is a choice of $N_K(m)$ that ensures there exists a function $\bar{\Kbf}(\xbf,\xbf')$ such that
    \begin{align}
        \lim_{m \rightarrow \infty} \sup_{\xbf,\xbf' \in \mathcal{X}} | \Kbf(\xbf,\xbf')-\bar{\Kbf}(\xbf,\xbf')|=0.
    \end{align}
    In other words, the analytic QNTK converges to a limiting QNTK as $m \rightarrow \infty$; this assumption can be satisfied by a wide range of practically relevant quantum circuit architectures \cite{Girardi2025}. 
\end{restatable}

Additionally, our quantum neural contextual bandit algorithms  rely on the following key property of the empirical QNTK: The empirical QNTK $\hat{\Kbf}_{{\thetabf}_0}$ converges to the analytic mean $\Kbf$ as the number of qubits $m$  grows. This convergence is guaranteed if the QNN architecture satisfies certain structural conditions.  The key QNN properties that influence the convergence are the \emph{past light cone} $\mathcal{N}_k$, defined as the set of parameters $\{{\thetabf}_i\}$ that can influence the local observable $f_k$, with $|\mathcal{N}| = \max_k |\mathcal{N}_k|$ and the \emph{future light cone} $\mathcal{M}_i$, the set of observables $\{f_k\}$ that a parameter ${\thetabf}_i$ can influence, with $|\mathcal{M}| = \max_i |\mathcal{M}_i|$. Note that $L$, $|\mathcal{M}|$, and $|\mathcal{N}|$  may depend on the width $m$. Formally, the structural conditions of the architecture and the convergence property of the empirical QNTK can be stated as follows:

\begin{restatable}{assumption}{circuitstruct}\label{circuitstruct}
The QNN architecture with $m$ qubits satisfies the following:
\begin{align} \lim_{m\rightarrow \infty} \frac{Lm |\mathcal{M}|^4|\mathcal{N}|^2}{N(m)^4}=0, \label{eq:condition} \end{align} 
\end{restatable}

\begin{theorem}[Theorem 4.13 of \cite{Girardi2025}] \label{thm:cvgQNTK}
Under Assumption \ref{circuitstruct}, when the QNN is randomly initialized, i.e., the parameters $\thetabf$ are
independent random variables, the empirical QNTK converges in probability to the analytic
QNTK as $m \rightarrow \infty$. In particular,  there exists a constant $c>0$ such that, for any $\xbf,\xbf' \in \mathcal{X}$, we
have \begin{align}
  \mathbb{P} \left( \Bigl|\hat{\Kbf}_{\thetabf_0}(\xbf,\xbf')- {\Kbf}(\xbf,\xbf')\Bigr| > \varepsilon \right) \leq \exp \left(-c \varepsilon^2 \frac{N_K(m)^2N(m)^4}{Lm |\mathcal{M}|^4|\mathcal{N}|^2}\right).
\end{align}   

\end{theorem} 

Note that the scaling of $|\mathcal N|$ and $|\mathcal M|$ with respect to the width $m$ and depth $L$ is dictated by the specific connectivity structure of the quantum ansatz. For instance, geometrically local circuits\footnote{Circuits where only neighbouring qubits interact \cite{abedi2023quantum}.} restrict light cone growth primarily to the circuit depth, whereas all-to-all architectures allow them to expand rapidly to cover the entire system size $m$.
Furthermore, it can be verified  that for circuits satisfying the conditions in Theorem~\ref{thm:cvgQNTK} and Assumption~\ref{ass:cvgtolimqntk}, $\Omega(1) \leq N_K(m) < O(|\mathcal N|)$ (see Lemma 4.9 in \cite{Girardi2025}).

Throughout this work, we consider QNN circuit architectures that satisfy \eqref{eq:condition}.
In particular, this  assumption is satisfied by the following 
QNN architectures:
\begin{itemize}
\item Constant-Depth Circuits ($L=O(1)$): For geometrically local circuits with a fixed depth $L$, the light cone sizes $|\mathcal{M}|$ and $|\mathcal{N}|$ are $O(1)$. Assuming the normalization $N(m) = \Omega(\sqrt{m})$, the convergence condition in~\eqref{eq:condition} is satisfied.

\item Logarithmic-Depth Circuits ($L=O(\log m)$). As established in \cite{Girardi2025}, allowing the circuit depth to grow with the number of qubits is a necessary condition for achieving quantum advantage. With a normalization of $N(m) = \Omega(\sqrt{m})$, it is straightforward to verify that our theoretical assumptions are satisfied by logarithmic-depth architectures. Section 2.5 of \cite{Girardi2025} provides specific examples of circuit constructions that meet these criteria.

\item Polynomial-Depth Circuits ($L=O(m^{\alpha})$ for some small $\alpha>0$). 
In our no-training regime, one might explore even larger (e.g., polynomial) depth scalings to further boost expressivity and potential for achieving more significant quantum advantage. However, this increased depth comes with additional trade-offs, including concerns about light-cone growth, QNTK concentration, and kernel scaling. A detailed discussion of these implications is provided in Section~\ref{sec:comparison_w_other_models}.

\end{itemize}

\section{Quantum Neural Tangent Kernel-Based UCB Algorithm} \label{qntkucbsec}
In this section, we introduce Quantum Neural Tangent Kernel (QNTK)-UCB, a new algorithm for contextual bandits based on QNTK. Although QNNs are universal function approximators \cite{goto2021universal,gonon2025universal}, training them via gradient descent is notoriously difficult due to the ``barren plateau'' phenomenon \cite{mcclean2018barren}, where network gradients vanish exponentially with the number of qubits.
As a result, directly  training the QNN model $f(\xbf; {\thetabf})$ to approximate the unknown reward function $h(\xbf)$ becomes computationally prohibitive and unstable.
To circumvent this, in our proposed algorithm, we freeze the QNN at a random initialization and utilize its associated Quantum Neural Tangent Kernel (QNTK) for regression.


\subsection{Algorithm Description}

Our proposed algorithm, \emph{QNTK-UCB}, is a kernelized UCB policy where the kernel is the empirical QNTK (defined in \eqref{eq:qntk-empirical}) induced by a randomly initialized QNN. The core advantage is that it bypasses the gradient-based training of the QNN, thereby circumventing the optimization difficulties posed by the barren plateau problem, while preserving the expressivity and inductive bias inherent in quantum neural networks.


The validity of this kernelized approach rests on the ``lazy training" phenomenon observed in over-parameterized networks. To this end, we first establish (see Lemma~\ref{lemma:realizable}) that for sufficiently wide QNNs, the reward function $h(\xbf)$ is realizable as a linear function in the tangent feature space. Specifically, there exists a parameter vector $\thetabf^*$ such that for all contexts $\xbf^i$, the reward is well-approximated by the first-order Taylor expansion around the initialization $\thetabf_0$: 
\begin{align}h(\xbf^i)= f(\xbf^i;\thetabf_0)+\langle \nabla_{\thetabf}f(\xbf^i;\thetabf_0),\thetabf^*-\thetabf_0\rangle. \label{eq:linearized_model}\end{align} Note that $\bm{\theta}_0$ is the parameter vector drawn from a random initialization distribution (e.g., each element in $\bm{\theta}_0$ is uniform over the interval $[0,\pi]$). 

We treat the randomly initialized QNN as a static feature extractor and define the $p$-dimensional feature map $\phibf: \mathcal{X} \to \mathbb{R}^p$ as the scaled gradient of the QNN model evaluated at ${\thetabf}_0$:
\begin{equation}\label{eq:feature-map}
\phibf(\xbf) := \frac{1}{\sqrt{N_K(m)}}\nabla_{\thetabf} f(\xbf;{\thetabf} )\big|_{{\thetabf}={\thetabf}_0}.
\end{equation}
This construction ensures that the inner product in feature space recovers the empirical QNTK, i.e., $\phibf(\xbf)^\top \phibf(\xbf') = \hat{\Kbf}_{\bm{\theta}_0}(\xbf, \xbf')$.

Our algorithm then proceeds as an instance of kernelized bandit \cite{valko2013finite} on this explicit feature space.  
At each round $t$, the agent maintains a regularized design matrix $\Zbf_{t-1}$ and a reward-weighted feature vector $\bbf_{t-1}$ which are defined as 
$$\Zbf_{t-1} = \lambda \Ibf + \sum_{\tau=1}^{t-1} \phibf(\xbf_{\tau, a_\tau}) \phibf(\xbf_{\tau, a_\tau})^\top 
\quad\mbox{and}\quad \bbf_{t-1} = \sum_{\tau=1}^{t-1} r_{\tau,a_{\tau}} \phibf(\xbf_{\tau, a_\tau}),$$
where $\lambda > 0$ is the regularization parameter.
The unknown linear parameter $\bm{\theta}^*$ 
is estimated via ridge regression as follows: $\hat{\bm{\theta}}_{t-1} = \Zbf_{t-1}^{-1} \bbf_{t-1} +\thetabf_0$.  
To balance exploration and exploitation, the agent selects the arm that maximizes a certain Upper Confidence Bound:
$$a_t = \operatorname*{arg max}_{a \in [K]} \left\{ \phibf(\xbf_{t,a})^\top (\hat{\bm{\theta}}_{t-1}-\thetabf_0) + \beta_{t-1} \sqrt{\phibf(\xbf_{t,a})^\top \Zbf_{t-1}^{-1} \phibf(\xbf_{t,a})} \right\}.$$
Here, $\beta_{t-1}$ is an exploration radius that controls the confidence width. Its precise value is derived from the regret analysis in Section \ref{sec:reg_analysis}. The complete procedure is summarized in Algorithm~\ref{alg:qntk-ucb}.

\begin{algorithm}[h]
\caption{QNTK-UCB Algorithm}\label{alg:qntk-ucb}
\begin{algorithmic}[1]
\State \textbf{Input:} Regularization parameter $\lambda > 0$, exploration parameter $\nu$, confidence parameter $\delta \in (0, 1)$, norm parameter $S$, number of rounds $T$.

\State \textbf{Initialization:}
\State Randomly initialize QNN parameters ${\thetabf}_0$.
\State Define feature map $\phibf(\xbf) = \frac{1}{\sqrt{N_K(m)}}\nabla_{{\thetabf}} f(\xbf;{\thetabf}_0)$.
\State Initialize $\Zbf_0 = \lambda \Ibf_p$ (where $p = \dim({\thetabf}_0)$) and $\bbf_0 = \mathbf{0} \in \mathbb{R}^p$. 

\For{$t=1, 2, \dots, T$}

\State Observe contexts $\{\xbf_{t,a}\}_{a=1}^K$.
\State Compute ridge regression estimate: $\hat{{\thetabf}}_{t-1} = \Zbf_{t-1}^{-1} \bbf_{t-1} +\thetabf_0 $.

\State Compute exploration radius $\beta_{t-1} = \nu\sqrt{\log\frac{\det(\Zbf_{t-1})}{\det(\lambda \Ibf)} + 2\log\big(\frac{1}{\delta}\big)}
+\sqrt{\lambda}S.$

\For{each arm $a \in [K]$}
\State Compute features: $\phibf_{t,a} = \phibf(\xbf_{t,a})$.
\State Compute predicted reward: $\hat{h}_{t-1}(\xbf_{t,a}) = \phibf_{t,a}^\top (\hat{{\thetabf}}_{t-1}-\thetabf_0)$.
\State Compute  width of confidence bound: $w_{t,a} = \beta_{t-1} \sqrt{\phibf_{t,a}^\top \Zbf_{t-1}^{-1} \phibf_{t,a}}$.
\State Compute UCB: $U_{t,a} = \hat{h}_{t-1}(\xbf_{t,a}) + w_{t,a}$.
\EndFor

\State Select action: $a_t = \arg\max_{a \in [K]} U_{t,a}$.
\State Observe reward $r_t = r_{t,a_t}$.
\State Update $\Zbf_t = \Zbf_{t-1} + \phibf_{t,a_t} \phibf_{t,a_t}^\top$.
\State Update $\bbf_t = \bbf_{t-1} + r_t \phibf_{t,a_t}$.
\EndFor

\end{algorithmic}
\end{algorithm}

\paragraph{Comparison with Existing Algorithms. }
The QNTK-UCB algorithm can be viewed as a quantum counterpart to   neural bandit algorithms \cite{neuralucb}, yet it possesses distinct operational and theoretical characteristics:
\begin{itemize}
    \item The classical NeuralUCB algorithm \cite{neuralucb} trains a classical neural network at regular intervals using gradient descent. While effective, applying this directly to quantum circuits is problematic due to the barren plateau phenomenon, where gradients vanish exponentially with system size, making training unstable or impossible. 

    QNTK-UCB circumvents this issue entirely by utilizing the {\em   fixed geometry} of the quantum feature space at initialization. By treating the quantum circuit as a static kernel rather than a trainable model, our approach requires no gradient updates during the bandit interaction, ensuring algorithmic stability while retaining the expressivity of the quantum ansatz. We discuss the implications of this fixed geometry on parameter efficiency in Section \ref{sec:comparison_w_other_models}.
    
    
    \item Although our algorithm shares a similar algebraic structure with KernelUCB \cite{kernelucb}, the key difference lies in the kernel employed. Standard classical kernels, such as RBF and Matérn, often struggle to accurately model quantum reward functions. In contrast, the QNTK  explicitly incorporates the inductive bias of quantum circuits, accounting for features like entanglement structure and data encoding methods. This enables QNTK-UCB to learn effectively in ``quantum-native'' environments, such as when analyzing the properties of quantum states, where classical kernels fail to capture essential correlations. This will be made more apparent in the experimental results. 

\end{itemize}

\subsection{Regret Analysis} \label{sec:reg_analysis}

In this section, we provide   theoretical guarantees for our QNTK-UCB algorithm. Our analysis relies on the concentration of the empirical QNTK to its limit, allowing us to bound the regret in terms of the quantum effective dimension. 

In the following, we overload the notation $\Kbf$ (and likewise $\hat{\Kbf}$ and $\bar{\Kbf}$) to denote either the {\em kernel function} $\Kbf(\xbf,\xbf')$ or the corresponding  {\em kernel matrix} $\Kbf = [\Kbf_{ij}]$, whose $(i,j)$-th element $\Kbf_{ij}$ is the kernel function evaluated at the $i$-th and $j$-th data points, i.e.,  $\Kbf_{ij} =\Kbf(\xbf^i, \xbf^j)$. Equipped with this notation, we first define the quantum effective dimension.

\begin{definition}\label{def:q-eff-dim}
  The {\em quantum effective dimension} $\widetilde{d}_{\mathrm{q}}(\lambda)$ of the quantum neural tangent kernel on the dataset $\mathcal{X}_{1:TK}$ is defined as: 
$$\widetilde{d}_{\mathrm{q}}(\lambda) = \frac{\log \det(\Ibf_{TK} + \bar{\Kbf} / \lambda)}{\log(1 + TK / \lambda)},$$ where $\bar{\Kbf}$ is the limiting QNTK defined in Assumption~\ref{ass:cvgtolimqntk}. 
\end{definition}

Intuitively, the quantum effective dimension measures the ``capacity'' of the feature space relative to the available data. When it is clear from context, we omit the dependence of $\widetilde{d}_{\mathrm{q}}(\lambda)$ on $\lambda$ and simply write $\widetilde{d}_{\mathrm{q}}$. 
This definition mirrors the effective dimension used in classical kernel bandits \cite{neuralucb} and intuitively measures the complexity of the feature space defined by the QNTK.


In addition to the  structural assumptions pertaining to  QNNs in Assumptions~\ref{fcvgtogp} and~\ref{ass:cvgtolimqntk},  
we make the following assumptions that are mild and standard in kernel bandit literature~\cite{cao2019,neuralucb}: 
\begin{restatable}{assumption}{bndx} \label{ass1}
The context vectors $\xbf_{t,a}$ satisfy $\|\xbf_{t,a}\|_2 = 1$ for all $t \in [T]$ and $a \in [K]$.
\end{restatable}
In fact, this can be assumed without loss of generality, by  normalizing the context vectors appropriately.
\begin{restatable}{assumption}{pdKlim}\label{ass2}
    The limiting QNTK matrix $\bar{\Kbf}$ (evaluated on the dataset $\mathcal{X}_{1:TK}$) is   positive definite, with a minimum eigenvalue $\lambda_0 > 0$, i.e. $\bar{\Kbf} \succeq \lambda_0 \Ibf$.
\end{restatable}

\paragraph{Main Result}

\begin{restatable}{theorem}{thmregret} \label{thm:regret}
Fix any $\delta\in(0,1)$. Let $m = \bigOmega\left( \frac{(TK)^3}{\lambda^2}\log\left(\frac{(TK)^2}{\delta}\right) \right)$. Then, 
with probability at least $1-\delta$, the cumulative regret of QNTK-UCB  (Algorithm \ref{alg:qntk-ucb}) satisfies
\begin{align*}
R_T&    := \sum_{t=1}^T \left[ h(\xbf_{t, a_t^*}) - h(\xbf_{t, a_t}) \right] \\*
&\le
3\sqrt{T}\sqrt{\widetilde{d}_{\mathrm{q}}\log \Big(1+\frac{TK}{\lambda}\Big)+1}\Big(\nu\sqrt{\widetilde{d}_{\mathrm{q}}\log\!\Big(1+\frac{TK}{\lambda}\Big)+ 1 +2\log\big(\frac{1}{\delta}\big)}+\sqrt{\lambda}S\Big),
\end{align*}
where $S \geq  \sqrt{2\hbf^\top \bar{\Kbf}^{-1} \hbf}$, $\hbf = [h(\xbf^1), \dots, h(\xbf^{TK})]^\top$,  and $\lambda \geq \max\{1,S^{-2}\}$. 
Ignoring logarithmic terms and constants,   this bound simplifies to $$ 
R_T = \tilde{\mathcal{O}}\left(\widetilde{d}_{\mathrm{q}} \sqrt{T} \right)
$$
\end{restatable}


\paragraph{Proof Sketch.} 
The proof (detailed in Appendix \ref{secA}) proceeds in three steps:
\begin{enumerate}

\item \textbf{Concentration and Quantum Linear Realizability.} We build upon recent findings regarding the behavior of Gaussian processes in quantum circuits \cite{Girardi2025} to control the concentration of the empirical QNTK $\hat{\Kbf}$ around the limiting $\bar{\Kbf}$. 
A technical innovation here is utilizing the concentration of measure on the unitary group  to bound the spectral distance $\|\hat{\Kbf} - \bar{\Kbf}\|_{\mathrm{F}}$ purely as a function of circuit architecture, independent of optimization dynamics (Lemma~\ref{lemma:conc_mtx_error}).
This convergence enables us to effectively ``linearize'' the QNN. 
Specifically, Lemma~\ref{lemma:realizable} guarantees the existence of a parameter vector $\bm{\theta}^* \in \mathbb{R}^p$ such that $h(\xbf)= \langle \nabla_{\thetabf}f(\xbf;\thetabf_0),\thetabf^*-\thetabf_0\rangle$ for all $\xbf \in \mathcal{X}_{1:TK}$, with the constraint that $\|\bm{\theta}^*\|_2 \le S$. 

\item \textbf{Confidence Ellipsoid and Instantaneous Regret.} 
We subsequently construct a confidence ellipsoid for the unknown parameter $\bm{\theta}^*$ and utilize the self-normalized martingale inequality for vector-valued martingales (Lemma \ref{lem:selfnorm}). 
While standard proofs for  classical NeuralUCB \cite{neuralucb} must bound the drift of the NTK during gradient descent to ensure the confidence sets remain valid (often requiring prohibitive width scaling), our analysis exploits the static geometry of the frozen quantum ansatz. Since the parameters are fixed at initialization, the kernel exhibits no drift, allowing us to guarantee that,  
with high probability, the true parameter resides within a bounded region centered around the ridge regression estimate.


\item \textbf{Determinant Bound and Total Regret.} Finally, we establish a bound on the cumulative regret by relating it to the sum of predictive variances, which is governed by the log-determinant of the kernel matrix (Lemma \ref{lem:elliptical}). This allows us to introduce  the \emph{quantum effective dimension} $\widetilde{d}_{\mathrm{q}}$  in Lemma \ref{lem:logdet-qntk}, demonstrating that the regret primarily scales with $\widetilde{d}_{\mathrm{q}}$. This dimension effectively captures the distinctive inductive bias of the quantum ansatz, thereby formally connecting the spectral decay of the specific quantum architecture to the learning efficiency of the algorithm.
\end{enumerate}


\begin{corollary}
    Under the same conditions as Theorem~\ref{thm:regret}, the expected cumulative regret satisfies $$
    \bar{R}_T \le
3\sqrt{T}\sqrt{\widetilde{d}_{\mathrm{q}}\log \Big(1+\frac{TK}{\lambda}\Big)+1}\Big(\nu\sqrt{\widetilde{d}_{\mathrm{q}}\log\!\Big(1+\frac{TK}{\lambda}\Big)+ 1 +2\log T}+\sqrt{\lambda}S\Big) +1 = \tilde{\mathcal{O}}\left(\widetilde{d}_{\mathrm{q}} \sqrt{T} \right)
    $$
\end{corollary}
\begin{proof}
    Follows from Theorem~\ref{thm:regret} by setting $\delta=1/T$.
\end{proof}

\subsection{Discussion and comparison with classical methods} \label{sec:comparison_w_other_models}

The regret guarantee in Theorem \ref{thm:regret} provides a foundation for analyzing the utility of QNTK-UCB. While the algebraic form of the regret bound mirrors that of standard kernelized bandits, the specific properties of the QNTK introduce distinct advantages in terms of parameter efficiency, inductive bias, and implicit regularization.

\paragraph{Regret and Parameter Efficiency} 



A significant implication of Theorem \ref{thm:regret} pertains to the model size needed to achieve the specified guarantees. Classical approaches, such as NeuralUCB \cite{neuralucb}, necessitate that the neural network operate within the ``lazy training'' or ``linear NTK'' regime to maintain the theoretical validity of the regret bound. In this regime, the weights are only adjusted  minimally from their initialization, ensuring that the empirical kernel remains effectively static.

However, achieving this regime imposes stringent constraints on model size. Notably, the network width $w$ must be exceedingly large to minimize the approximation error between the neural network and its linearized kernel. The analysis of NeuralUCB (see Lemma 5.1 and Lemma 5.4 in~\cite{neuralucb}) demonstrates that the width must scale as a high-order polynomial of the horizon, specifically $w = \tilde{\Omega}\left((TK)^6\right)$. Since the number of parameters $p$ in a fully connected network scales quadratically with the width, this results in a prohibitively high parameter requirement of:
$$p_{\mathrm{c},\mathrm{train}}=\tilde{\Omega}\left((TK)^{12}\right).$$  

Even for the static baseline NeuralUCB0 (or Classical NTK), which relies on a fixed NTK at initialization, a significant degree of over-parameterization is still necessary. This method adheres to the conditions outlined in Lemma 5.1 of \cite{neuralucb}, which requires the width to scale as $w = \tilde{\Omega}((TK)^4)$, leading to a parameter requirement of:
$$p_{ \mathrm{c}, \mathrm{no}\text{-}\mathrm{train}}=\tilde{\Omega}\left((TK)^{8}\right).$$

In contrast, our QNTK-UCB framework achieves similar regret guarantees with a markedly more efficient parameter scaling. As indicated in Theorem \ref{thm:regret}, the conditions for our bound hold if: $$p_{ \mathrm{q}, \mathrm{no}\text{-}\mathrm{train}} =\tilde{\Omega}\left((TK)^3\right).$$
This substantial difference underscores the advantage of quantum models in terms of model compactness. By utilizing the high-dimensional Hilbert space of a relatively small quantum circuit (with small \(p\)), QNTK-UCB offers a robust, mathematically guaranteed kernel regime without the excessive over-parameterization required to linearize classical deep networks.

\paragraph{Inductive Bias and Representational Power.}
In addition to enhancing efficiency, the QNTK  introduces a distinctive inductive bias. Classical kernels, such as RBF and Mat\'ern, or standard NTKs tend to favor classically smooth functions. In contrast, the QNTK is fundamentally shaped by the quantum circuit architecture, particularly its entanglement structure and the data encoding map $V(\xbf)$. 

QNNs project classical inputs into an exponentially large Hilbert space with dimension $2^m$. This high-dimensional embedding enables the QNTK to capture correlations that reflect the inherent properties of quantum mechanical processes, which can be challenging for classical models. As a result, we anticipate that QNTK-UCB will surpass classical benchmarks in ``quantum-native'' bandit tasks, such as optimizing variational quantum eigensolvers (VQEs) or classifying phases of matter, where the underlying reward function exhibits symmetries consistent with the quantum circuit. Our experimental results support this expectation.

Within the contextual bandit framework, the benefits of this quantum inductive bias can be understood via the spectral characteristics of the QNTK Gram matrix of the observed contexts. When the inductive bias of the QNTK aligns effectively with the ground truth reward function $h$, it results in a faster decay   of the eigenvalues compared to generic, isotropic classical kernels (e.g., RBF), which tend to disperse probability mass across the feature space. This sharper spectral decay minimizes the sum defining $\widetilde{d}_{\mathrm{q}}$, effectively narrowing the width of the confidence ellipsoid. Consequently, the algorithm requires significantly fewer samples to reduce the posterior variance below the optimality gap, thereby decreasing the sample complexity of exploration. 

\paragraph{Information gain and effective dimension.}

To interpret the regret bound in Theorem \ref{thm:regret}, it is useful to rewrite the log-determinant term as a kernel information gain quantity. Recall that the QNTK feature map $\phibf(\xbf) = \frac{1}{\sqrt{N_K(m)}} \nabla_{\bm{\theta}} f(\xbf;\bm{\theta}_0)$ induces the limiting Gram matrix $\bar{\Kbf} \in \mathbb{R}^{TK \times TK}$. The quantity $$\gamma_T^{(\mathrm{q})} := \log \det \left( \Ibf + \frac{1}{\lambda} \bar{\Kbf} \right)$$
is the standard information gain term appearing in kernel bandit analyses \cite{Srinivas_2012}. Our quantum effective dimension is exactly its normalized version: $$\widetilde{d}_{\mathrm{q}} = \frac{\gamma_T^{(\mathrm{q})}}{\log(1 + TK/\lambda)}.$$
Thus, the regret bound in Theorem \ref{thm:regret} is controlled by $\gamma_T^{( \mathrm{q})}$, or equivalently $\widetilde{d}_{\mathrm{q}}$.

In  classical NeuralUCB analysis \cite{neuralucb}, the same structure appears with the classical limiting NTK Gram matrix $\mathbf{H}$ \cite{ntk} in place of $\bar{\Kbf}$. The classical bound depends on: $$\gamma_T^{(\text{c})} := \log \det \left( \Ibf + \frac{1}{\lambda} \mathbf{H} \right) \quad\mbox{and}\quad \widetilde{d}_{\text{c}} = \frac{\gamma_T^{(\text{c})}}{\log(1 + TK/\lambda)}.$$
Our theorem demonstrates that, once we pass to the kernelized (training-free) regime, the distinction between ``quantum" and ``classical" enters primarily through the spectrum of the corresponding limiting kernel matrix, namely, $\bar{\Kbf}$ versus $\mathbf{H}$. Consequently, whenever the eigenvalues of $\bar{\Kbf}$ decay faster than those of $\mathbf{H}$ on the realized context sequence, we obtain a smaller information gain $\gamma_T^{(\mathrm{q})}$ and hence a tighter regret guarantee.

The behavior of $\widetilde{d}_{\mathrm{q}}$ highlights a quantum-specific trade-off regarding the ``barren plateau" phenomenon. In variational quantum algorithms, deeper or unstructured circuits often exhibit strong concentration-of-measure effects that manifest as exponentially vanishing gradients, making gradient-based training difficult. However, in our kernelized setting, this concentration plays a distinct, constructive role:
\begin{itemize}
    \item \textbf{Implicit Regularization via Gradient Scaling.}  The barren plateau phenomenon indicates that the magnitude of the gradients \( \|\nabla_{\bm{\theta}} f(\xbf; \bm{\theta}_0)\| \) vanishes exponentially as the number of qubits $m$ tends to infinity. Given that the empirical QNTK is based on the inner products of these gradients, this gradient concentration leads to a reduction in both the entries and eigenvalues of the Gram matrix \(\hat{\Kbf}\). Recent theoretical work~\cite{exp_conc_qntk} also demonstrates that high expressivity in QNNs leads to an exponential concentration of QNTK values toward zero. As a result, the term \(\gamma_T^{(\mathrm{q})} = \log\det(\Ibf + \bar{\Kbf}/\lambda) \) becomes significantly smaller compared to \(\gamma_T^{(\mathrm{c})}\),  the information gain of  kernels derived from wide, non-concentrated classical networks. This spectral compression serves as a form of implicit regularization, which may reduce the regret upper bound. 

\item However, it is important to note that the spectral shrinkage induced by concentration does not automatically confer benefits. Consider the realizability constant $S$ introduced in Theorem~\ref{thm:regret}, which is influenced by the norm of the reward function in the RKHS in the following manner:  $S^2 \approx \hbf^\top \bar{\Kbf}^{-1} \hbf$. If the   concentration phenomenon results in the  scaling down of the entire kernel by a constant factor, the parameter norm $S$ consequently increases, counteracting the advantages of a lower effective dimension.

The true quantum advantage arises when concentration is non-uniform. An effective quantum architecture should demonstrate Kernel-Target Alignment; that is, it should retain large eigenvalues along the specific directions that align with  the reward function \(\hbf\), ensuring that \(\hbf^\top \bar{\Kbf}^{-1} \hbf\) remains uniformly upper bounded. Meanwhile, it should concentrate significantly along the majority of orthogonal, ``irrelevant'' directions. In this scenario, the quantum effective dimension   $\widetilde{d}_{\mathrm{q}}$  decreases rapidly as the tail of the spectrum vanishes, while the realizability constant \(S\) stays small because the signal direction is preserved. This selective spectral decay is what enables QNTK-UCB to outperform classical baselines.

\end{itemize}



\section{Experiments}
To validate our theoretical findings, we compare the performance of QNTK-UCB against state-of-the-art classical neural bandit algorithms. Our experiments are designed to test the hypothesis that quantum kernels provide a superior inductive bias and higher parameter efficiency for non-linear reward functions.

\subsection{Gaussian Quantiles} \label{sec:exp_gaussQ}

\begin{figure*}[tbp]                
  \centering

  \begin{subfigure}{.33\textwidth}
    \includegraphics[width=\linewidth]{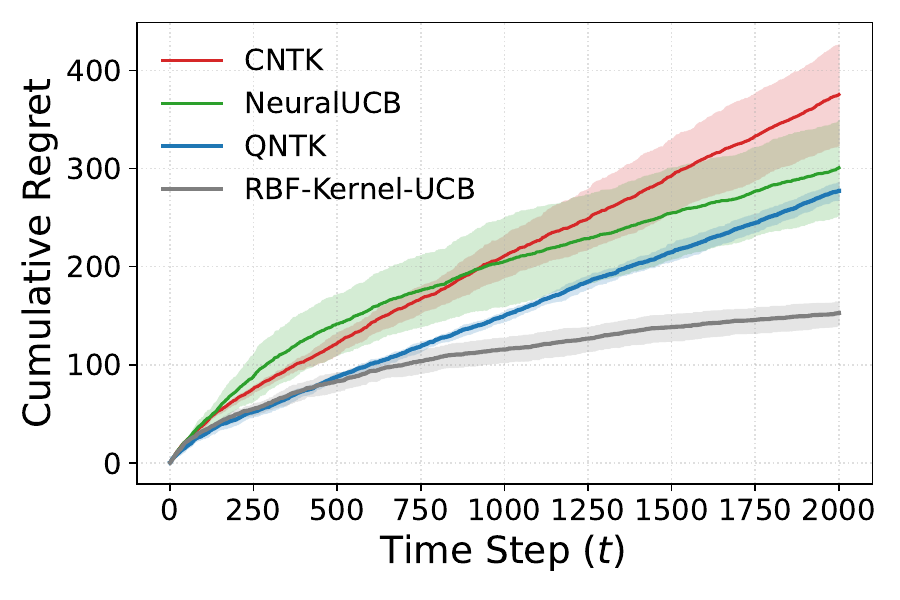}
    \caption{$m=3, p=36$}
  \end{subfigure}\hfill
  \begin{subfigure}{.33\textwidth}
    \includegraphics[width=\linewidth]{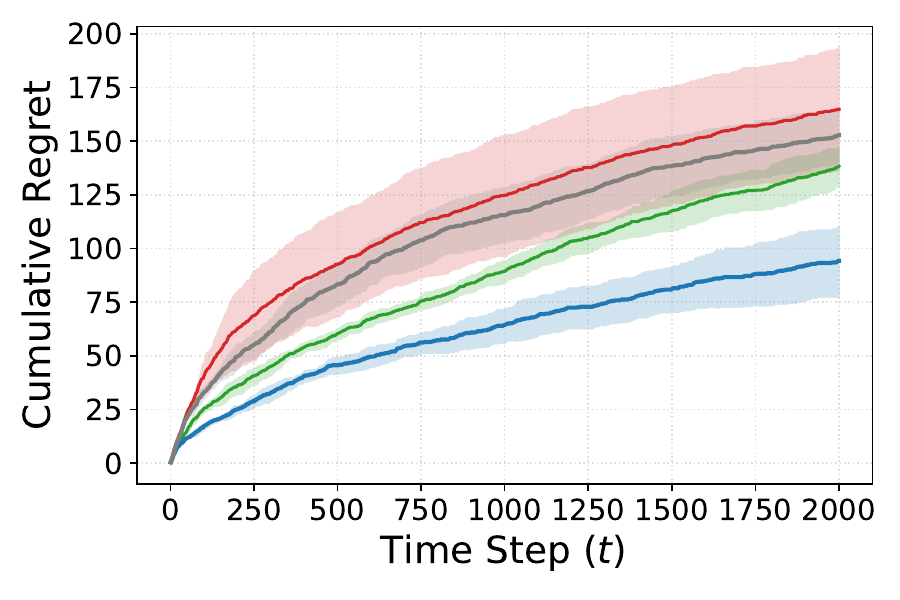}
    \caption{$m=5, p=60$}
  \end{subfigure}\hfill%
  \begin{subfigure}{.33\textwidth}
    \includegraphics[width=\linewidth]{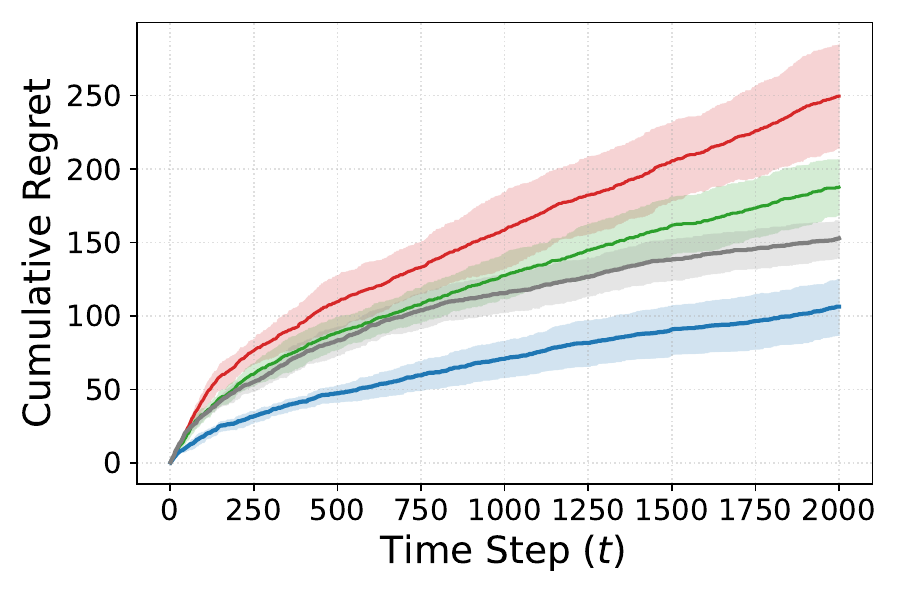}
    \caption{$m=10, p=120$}
  \end{subfigure}

  \caption{Bandit task with reward from Gaussian Quantile Classification 
  }
  \label{fig:GuassQ}
\end{figure*}

We consider a $K$-armed contextual bandit problem with $K=2$, where the reward function is defined by a non-linear decision boundary in $\mathbb{R}^d$. The base feature vectors, denoted as $\xbf_t$, are sampled from a multi-dimensional Gaussian distribution. The true class labels, $y_t \in \{0, 1\}$, are assigned according to Gaussian quantiles. Geometrically, this configuration creates concentric hypershells within the feature space, resembling a non-linear binary classification task akin to distinguishing between ``circles'' or ``spheres''.


At each round $t$, the agent receives a context vector $\xbf_t \in \mathbb{R}^d$. To model the arm-specific rewards, we employ the standard disjoint context encoding \cite{li2010contextual}: the agent observes a set of arm features $\{\xbf_{t,a}\}_{a \in \{0,1\}}$ where $\xbf_{t,0} = [\xbf_t, \mathbf{0}]$ and $\xbf_{t,1} = [\mathbf{0}, \xbf_t]$. The agent selects an arm $a_t$ and receives reward $r_t = 1$ if $a_t$ matches the true class label $y_t$, and $r_t = 0$ otherwise. 

We compare the following four benchmark algorithms:
\begin{itemize}
    \item QNTK (Ours): Uses the empirical quantum neural tangent kernel derived from a Strongly Entangling Layers ansatz with $L=4$ layers and varying number of qubits $m \in \{3, 5, 10\}$. The number of trainable parameters in this ansatz is $p = 3mL$.
    \item NeuralUCB: A classical neural contextual bandit that trains a Multi-Layer Perceptron (MLP) via gradient descent. The network consists of one hidden layer and uses the ReLU activation function. Optimization is performed using Adam with a learning rate of $\eta = 0.01$. 
    \item CNTK (Classical NTK, or NeuralUCB0 \cite{neuralucb}): A kernelized UCB algorithm using the fixed empirical NTK of a randomly initialized classical MLP.
    \item RBF-Kernel-UCB \cite{valko2013finite}: Kernelized UCB algorithm using the RBF kernel.
\end{itemize}



To fairly evaluate parameter efficiency, we constrain the classical models (NeuralUCB and C-NTK) to have the same number of trainable parameters $p$  as their quantum counterparts. For a given quantum circuit with $p_{\mathrm{q}}$ parameters, we analytically adjust the width of the classical MLP such that its total parameter count $p_{\mathrm{c}}$ satisfies $p_{\mathrm{c}} \approx p_{\mathrm{q}}$. 
For all models, we performed a grid search to optimize the regularization parameter $\lambda \in \{0.01, 0.1, 1.0\}$ and the fixed exploration radius $\beta_t = \beta \in \{0.05,0.1,0.5,1.0,3.0\}$.

Figure \ref{fig:GuassQ} illustrates the cumulative regret averaged over 30 independent trials with $T=2000$. We analyze the performance across three distinct regimes of model complexity:
\begin{itemize}
    \item Under-Parameterized Regime ($m=3$, $p=36$): As shown in Fig. \ref{fig:GuassQ}(a), the model capacity is constrained. All algorithms exhibit relatively steep regret curves, indicating that the model is too simple to perfectly capture the non-linear decision boundary. However, QNTK still achieves lower regret than the classical baselines. This confirms our hypothesis on parameter efficiency: even with minimal number of qubits, the quantum feature map provides a richer representation than a classical network of equivalent size, allowing for better approximation of the reward function under strict resource constraints.
    
    \item Optimal Regime ($m=5$, $p= 60$): In Fig. \ref{fig:GuassQ}(b), the model size increases to an intermediate level. Here, we observe clear sublinear regret for all methods, indicating that the models are sufficiently expressive to learn the task. QNTK maintains a clear lead, demonstrating the superiority of the quantum inductive bias on this task.
    
    \item Over-Parameterized ($m=10$, $p= 120$): Fig. \ref{fig:GuassQ}(c) reveals a divergence in behavior. For the classical methods (NeuralUCB and C-NTK), the cumulative regret becomes steeper compared to the $m=5$ case. This degradation is expected in classical learning theory: as the parameter count $p$ increases, the model requires more data to converge, and the variance term in the regret bound (governed by the effective dimension) grows. On the other hand, QNTK remains robust, its regret for $m=10$ is very close to that of $m=5$, showing no signs of performance degradation or overfitting. This empirically validates the implicit regularization property of the QNTK discussed in Section \ref{sec:comparison_w_other_models}. While the classical effective dimension increase with $p$, the quantum effective dimension $\widetilde{d}_{\mathrm{q}}$ saturates due to the concentration of the kernel spectrum, rendering the quantum algorithm resilient to high model complexity. 
\end{itemize}

\begin{figure}[t!] 
  \centering 
    \includegraphics[width=.75\linewidth]{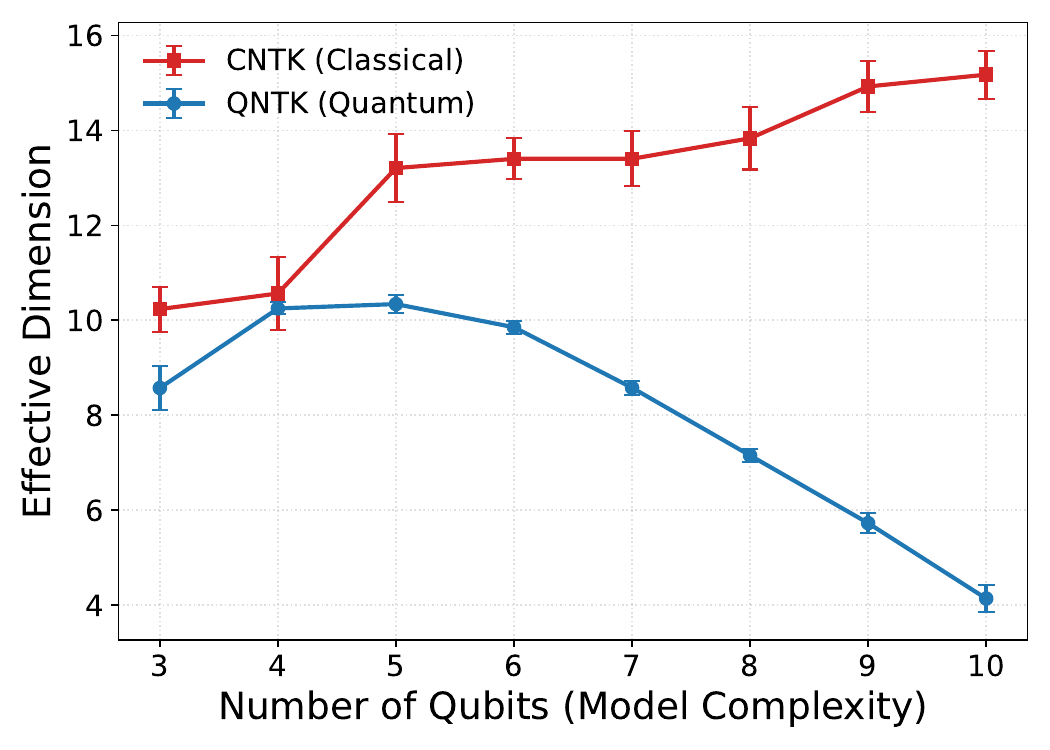}
    \caption{Change of Effective Dimension for Increasing Parameter size}
  \label{fig:circle_d_e}

\end{figure}

To further investigate the performance difference observed in Fig. \ref{fig:GuassQ}, we analyzed the effective dimension of the feature representations as function of the number of qubits (and hence the number of parameters). We plot the empirical feature dimension ($ \frac{\log \det(\mathbf{I} + \hat{\Kbf}/\lambda)}{\log(1 + TK/\lambda)}$) for $T=2000$ in Fig. \ref{fig:circle_d_e}. We observe two distinct behaviors that shed more light on our results.


The Classical NTK shown in red demonstrates a monotonic increase in effective dimension, consistently surpassing the QNTK. While this trend signifies high expressivity, the excessive dimensionality observed in the over-parameterized regime ($m=10$) suggests a potential ``over-spending'' of model capacity. This phenomenon directly correlates with the poorer regret performance illustrated in Fig.~\ref{fig:GuassQ}(c), where the classical model also experiences higher variance.

Conversely, the QNTK, depicted in blue, exhibits a general decreasing trend. The initial increase as $m$ rises from $3$ to $5$ reflects the necessary enhancement in expressivity required to effectively capture the non-linear decision boundary. Significantly, beyond $m=5$, the effective dimension reaches saturation and subsequently decreases. This behavior is indicative of the concentration of measure in the quantum feature space, wherein an intrinsic regularization mechanism curtails the quantum model's complexity, preventing it from growing unbounded with respect to parameter count.

\subsection{Online Quantum Initial State Recommendation}

\begin{figure}[t!] 
  \centering 
    \includegraphics[width=.75\linewidth]{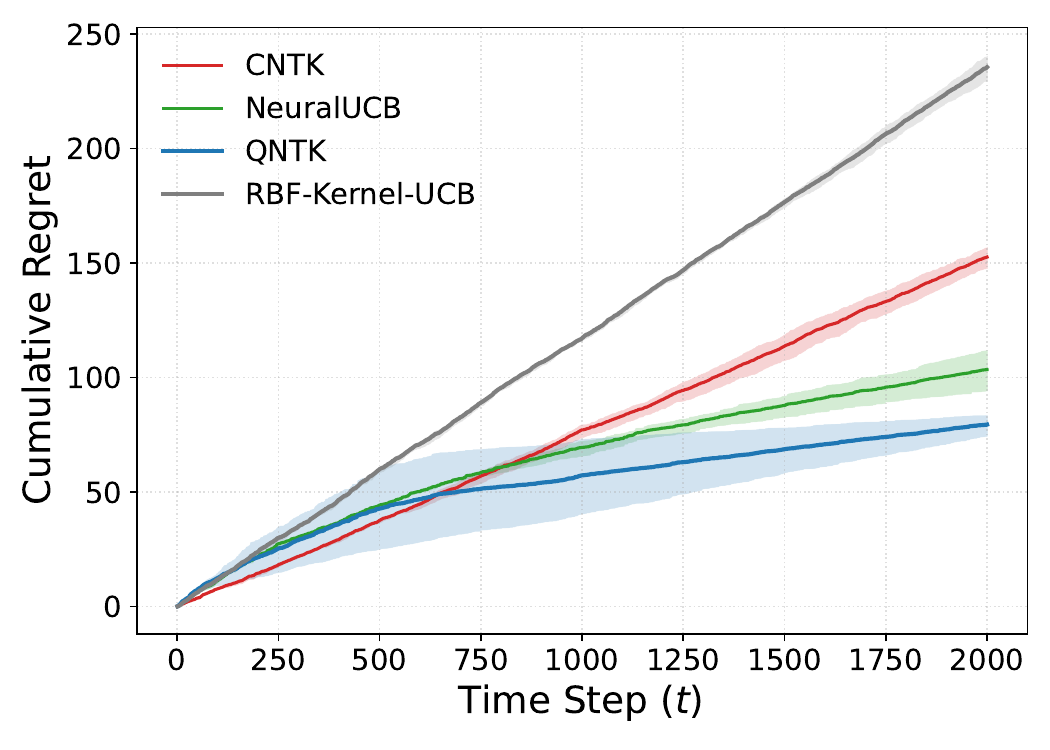}
    \caption{Bandit task with VQE optimization start point recommendation}
  \label{fig:vqe_recommender}

\end{figure}


To demonstrate the utility of QNTK in quantum-native tasks, we investigate the problem of identifying optimal initial states for Variational Quantum Eigensolvers (VQE) that find the ground state of a  Hamiltonian \cite{Brahmachari_2024}. 
VQE optimization is highly sensitive to initialization. We formulate the choice of initial VQE ansatz parameters as a bandit problem with side information (or contextual bandit problem), where the side information is given by the problem Hamiltonian.

We focus on a family of 4-qubit transverse field Ising Hamiltonians, $$H(c) = - \sum_{i=1}^{m-1} \sigma^Z_i \sigma^Z_{i+1} - c \sum_{i=1}^{m} \sigma^X_i,$$ where $\sigma^A_i$, for $ A\in \{X,Z\}$, denotes the Pauli-A operator acting on the $i$th qubit, and $c$ is the transverse field strength.  At each round $t$, the environment generates a field strength $c_t$ and constructs the corresponding Hamiltonian $H(c_t)$, which is revealed to the learner as a shared context. 

The learner has access to $K=5$ ``arms", each corresponding to a fixed distinct choice of initial state $|\psi_a\rangle$ for the VQE. At time $t$, the learner selects an arm $a_t \in [K]$ and the environment runs a short-depth VQE optimization initialized at $|\psi_{a_t}\rangle$, with a VQE ansatz $U(\bm{\theta})$ to approximate the ground state of $H(c_t)$. Here, we use a shallow 2-layer hardware-efficient ansatz $U(\bm{\theta})$ on $m=4$ qubits. The ansatz has 2 layers, each layer consists of per-qubit Euler rotations Rot$(\cdot)$ on all qubits followed by a linear entangling chain with CNOT. The energy objective is the  expectation $$E(\thetabf;c_t) = \langle \psi_{a_t}| U(\thetabf)^{\dagger} H(c_t) U(\thetabf) |\psi_{a_t}\rangle$$
and we perform $I=5$ gradient steps to obtain $\thetabf_I$. The resulting approximate ground state $\vert \psi(\thetabf_I|a_t,c_t)\rangle :=U(\bm{\theta}_I)|\psi_{a_t}\rangle$ depends on both the initial state and the Hamiltonian side information. 
The learner then observes the reward $$r_{t,a_t}=-{\rm Tr}(H(c_t)\vert \psi(\thetabf_I|a_t,c_t)\rangle \langle \psi(\thetabf_I|a_t,c_t) \vert) +\xi_t,$$ which is the negative final energy corrupted by Gaussian noise $\xi_t$ modeling the finite-shot measurement error. 


Similar  to  the previous experiment, we compare QNTK-UCB to NeuralUCB, CNTK-UCB (NeuralUCB0), and RBF-Kernel-UCB. To ensure a fair comparison, we ensured that the quantum and classical models have  the same number of parameters. Furthermore, we optimized the hyperparameters for all algorithms using the same grid search strategy as in Section~\ref{sec:exp_gaussQ};  selecting the regularization parameter $\lambda$ from the set $\{0.01,0.1,1.0\}$ and the fixed exploration radius $\beta$ from the set $\{0.05, 0.1, 0.5, 1.0, 3.0\}$. Figure~\ref{fig:vqe_recommender} displays the cumulative regrets of the various algorithms.

The QNTK agent clearly demonstrates superior performance compared to the classical baselines. The mapping from ``Hamiltonian parameter $c$" to ``Optimal Initial State" is governed by the underlying phase transition of the Ising model. The QNTK, being derived from a quantum circuit, naturally captures the correlations and symmetries of this Hilbert space landscape. On the other hand, the classical networks, lacking this specific inductive bias, require more samples to learn the mapping from Hamiltonian parameters to optimal ansatz initializations.

\section{Conclusion}\label{sec13}
We have introduced a new class of quantum-enhanced neural contextual bandit algorithms that not only achieve regret performance comparable to classical neural UCB methods but also do so with a significantly reduced number of model parameters. By leveraging recent advancements in QNTK theory, we derived a regret bound that scales as $\mathcal{O}(\tilde{d}_{\rm q}\sqrt{T})$, where $\tilde{d}_{\rm q}$ represents the effective dimension of the QNTK. This approach effectively operates within the QNTK regime, reducing to a kernelized model with a static quantum tangent kernel. As a result, we successfully navigated the challenges of barren plateaus and avoid the high computational costs associated with explicitly training QNNs  for contextual bandit applications.

Nonetheless, the reliance on a static kernel may constrain the model's expressivity when faced with complex reward functions. Future research could explore {\em   hybrid} quantum-classical models that balance quantum expressivity with classical trainability. Potential avenues include architectures that integrate quantum feature maps with trainable classical neural networks, as suggested by recent studies \cite{nakaji2023quantum}. 

Moreover, while our analysis underscores parameter reduction as a key source of quantum advantage, our framework also accommodates circuit depth that scales with the number of qubits. This flexibility enables the development of families of quantum circuits believed to be classically hard to simulate, thereby introducing a new class of uniquely quantum models whose dynamics cannot be efficiently replicated by classical algorithms. This positions our work not only as a significant step toward quantum efficiency in contextual bandits but also as a promising avenue for exploring quantum advantages that extend beyond mere parameter efficiency.


\backmatter





\bmhead{Acknowledgements}

STJ is supported through the Royal Society International Travel Grant (IES\textbackslash R2\textbackslash 242104) and partly through EPSRC Quantum Technologies Career Acceleration Fellowship.










\bibliography{sn-bibliography}

\begin{appendices}

\section{Proof of Theorem \ref{thm:regret}}\label{secA}

This section provides the proof for Theorem \ref{thm:regret}. First recall some definitions:

\begin{enumerate}
    \item Empirical QNTK, $\hat{\Kbf}$: The random neural tangent kernel computed from a single instance of a randomly initialized QNN (for a given fixed structure, at width $m$).
    $$ \hat{\Kbf}_{i j} = \frac{1}{N_K(m)} \nabla_\mathbf{{\thetabf}} f(\xbf^i;\thetabf_0)^\top \nabla_{\thetabf} f(\xbf^j;\thetabf_0). $$ We suppress the dependence of $\hat{\Kbf}$ on $\thetabf_0$. Unless stated otherwise, the empirical kernel is always evaluated at the random initialization $\thetabf_0$.

    \item Analytic QNTK, $\Kbf$: The expected value of the empirical kernel.
    $$ \Kbf_{i j} = \E_{\thetabf_0} \left[ \hat{\Kbf}_{i j} \right]. $$
    \item Limiting QNTK, $\bar{\Kbf}$: The limiting kernel that the analytic kernel converges to in the infinite-qubit limit.
    $$ \bar{\Kbf}_{i j} = \lim_{m\to\infty} \Kbf_{i j}. $$
\end{enumerate}

\noindent We make one additional assumption for notational convenience:

\begin{restatable}{assumption}{zero_output}\label{ass:zero_output}
The contexts are normalized so that $\|\xbf\|_2=1$ for all $\xbf\in\mathcal X$, and the QNN model is centered at initialization in the sense that
\begin{equation}\label{eq:zero_output}
f(\xbf;\thetabf_0)=0 \qquad \forall \xbf\in\mathcal X.
\end{equation}
\end{restatable}

Note that this assumption is mild and is imposed only to simplify the realizability statements. Indeed, given any model $f(\cdot;\thetabf)$ and initialization $\thetabf_0$, we can define the centered model
\[
\widetilde f(\xbf;\thetabf) := f(\xbf;\thetabf) - f(\xbf;\thetabf_0),
\]
which satisfies $\widetilde f(\xbf;\thetabf_0)=0$ for all $\xbf$ and has the same tangent features at initialization, i.e.,
$\nabla_{\thetabf}\widetilde f(\xbf;\thetabf)\big|_{\thetabf=\thetabf_0}
=
\nabla_{\thetabf} f(\xbf;\thetabf)\big|_{\thetabf=\thetabf_0}$.

\subsection{Realizability}

\begin{lemma} \label{lemma:a1}
For any $\varepsilon > 0$ and $\delta \in (0, 1)$, there exists a number of qubits $m_0$ and a QNN structure, such that for all $m \ge m_0$, with probability at least $1-\delta$ over the random initialization of ${\thetabf}_0$, we have:
$$ \abs{\hat{\Kbf}(\xbf, \xbf') - \bar{\Kbf}(\xbf, \xbf')} \le \varepsilon $$
for any pair of inputs $\xbf, \xbf' \in \mathcal{X}_{1:TK}$. 
\end{lemma}

\begin{proof}
Fix any $\xbf,\xbf' \in \mathcal{X}_{1:TK}$, We need the expected empirical kernel as the bridge, by triangle inequality,
$$ \big|\hat{\Kbf}(\xbf, \xbf') - \bar{\Kbf}(\xbf, \xbf')\big| \le \underbrace{\big|\hat{\Kbf}(\xbf, \xbf') - \Kbf(\xbf, \xbf')\big|}_{\text{Stochastic Part}} + \underbrace{\big|\Kbf(\xbf, \xbf') - \bar{\Kbf}(\xbf, \xbf')\big|}_{\text{Deterministic Part}}. $$

First bound the Stochastic Part: For any given $\varepsilon > 0$, by Assumption \ref{circuitstruct} and Theorem \ref{thm:cvgQNTK}, there exist some constant $c$ such that \begin{equation} \label{ineq:k_conc}
    \mathbb{P}\left( \abs{\hat{\Kbf}(\xbf, \xbf') - \Kbf(\xbf, \xbf')} \ge \frac{\varepsilon}{2} \right) \le \exp\left[-c\, \varepsilon^2N_K(m)^2\frac{N(m)^4}{Lm
|\mathcal M|^4|\mathcal N|^2}\right].
\end{equation}
Here, $m$ is the number of qubits, and $\mathcal{M},\mathcal{N},N_K(m),N(m)$ are QNN structure dependent parameters and normalization factors, as defined in \cite{Girardi2025}. Note for QNN structures satisfying Assumption \ref{circuitstruct}, the   expression on the right of \eqref{ineq:k_conc} decreases with  $m$; see more details in Lemma~\ref{lemma:qnn_structure}.

We then bound the Deterministic Part. By Assumption \ref{ass:cvgtolimqntk}, we have
$$ \lim_{m\to\infty} \sup_{\xbf,\xbf' \in \mathcal{X}} \abs{\Kbf(\xbf, \xbf') - \bar{\Kbf}(\xbf, \xbf')} = 0. $$
This means that for any given $\varepsilon > 0$, there exists a number of qubits $m_2$ such that for all $m \ge m_2$:
\begin{align}
    \abs{\Kbf(\xbf, \xbf') - \bar{\Kbf}(\xbf, \xbf')} \le \frac{\varepsilon}{2}. \label{eqn:eps2}
\end{align} 
Combining \eqref{ineq:k_conc} and \eqref{eqn:eps2} yields the desired result.
\end{proof}

\noindent Next, we prove the concentration of kernel matrix in the Frobenius norm. 

\begin{lemma} \label{lemma:conc_mtx_error}
For any $\varepsilon > 0$ and $\delta \in (0, 1)$, if the QNN architecture satisfies the scaling condition
$$ \frac{N_K(m)^2 N(m)^4}{Lm |\mathcal{M}|^4 |\mathcal{N}|^2} = \bigOmega\left( \frac{1}{\varepsilon^2}\log\left(\frac{(TK)^2}{\delta}\right) \right), $$
and $$m \geq C_{\varepsilon}\quad  \text{ such that}\quad  \sup_{\xbf,\xbf' \in \mathcal{X}} \abs{\Kbf(\xbf, \xbf') - \bar{\Kbf}(\xbf, \xbf')} \leq \frac{\varepsilon }{ 2}.$$
Then with probability at least $1-\delta$ over the random initialization of ${\thetabf}_0$, we have
$$ \norm{\hat{\mathbf{K}} - \bar{\mathbf{K}}}_{\mathrm{F}} \le TK\varepsilon. $$
\end{lemma}


\begin{proof}
Lemma \ref{lemma:a1} establishes that for any $\varepsilon' > 0$ and $\delta' \in (0, 1)$, if the number of qubits $m$ is sufficiently large, then
$$ \mathbb{P}\left(\abs{\hat{\mathbf{K}}_{ij} - \bar{\mathbf{K}}_{ij}} \ge \varepsilon'\right) \le \delta',$$
Here $\hat{\Kbf}_{i j}:=\hat{\Kbf}(\xbf^i,\xbf^j)$, similarly for $\bar{\Kbf}$. We want this bound to hold simultaneously for all $(TK)^2$ entries in the matrix with a total failure probability of at most $\delta$. We use the union bound: 
$$ \mathbb{P}\left(\exists (i,j) \in [TK]^2  \,:\,\abs{\hat{\mathbf{K}}_{ij} - \bar{\mathbf{K}}_{ij}} \ge \varepsilon\right) \le \sum_{i=1}^{TK}\sum_{ j=1}^{TK} \mathbb{P}\left(\abs{\hat{\mathbf{K}}_{ij} - \bar{\mathbf{K}}_{ij}} \ge \varepsilon\right). $$
Let the probability of failure for a single entry be $\delta' = \delta / (TK)^2$. From the   concentration bound (Eq.~\eqref{ineq:k_conc} in Lemma \ref{lemma:a1}), we need to satisfy:
$$ \exp \left( -c \frac{N_K(m)^2 N(m)^4}{Lm |\mathcal{M}|^4 |\mathcal{N}|^2} \varepsilon^2 \right) \le \frac{\delta}{(TK)^2}. $$
Taking logarithms and rearranging gives the following required scaling condition on the architecture:
$$ \frac{N_K(m)^2 N(m)^4}{Lm |\mathcal{M}|^4 |\mathcal{N}|^2} \ge \frac{1}{c\varepsilon^2} \log\left(\frac{(TK)^2}{\delta}\right),$$
which is precisely the first condition stated in this lemma. Together with the second condition, we have, with probability at least $1-\delta$, for all $(i,j)$:
$$ \abs{\hat{\mathbf{K}}_{ij} - \bar{\mathbf{K}}_{ij}} \le \varepsilon. $$

\noindent We can now bound the Frobenius norm of the difference
\begin{align*}
\norm{\hat{\mathbf{K}} - \bar{\mathbf{K}}}_{\mathrm{F}}^2 = \sum_{i=1}^{TK}\sum_{ j=1}^{TK}\abs{\hat{\mathbf{K}}_{ij} - \bar{\mathbf{K}}_{ij}} ^2 
\le\sum_{i=1}^{TK}\sum_{ j=1}^{TK} \varepsilon^2 
= (TK)^2 \varepsilon^2.
\end{align*}
Hence
$$ \norm{\hat{\mathbf{K}} - \bar{\mathbf{K}}}_{\mathrm{F}} \le TK\varepsilon, $$ as desired.
\end{proof}

\begin{lemma} \label{lemma:qnn_structure}
The assumption in Lemma \ref{lemma:conc_mtx_error} can be achieved with some QNN structure, with $m = \bigOmega\left( \frac{1}{\varepsilon^2}\log\left(\frac{(TK)^2}{\delta}\right) \right)$.
\end{lemma}

\begin{proof}
    We give an example QNN architecture in Fig.~\ref{fig:circuit}. Section 2.5 in \cite{Girardi2025} established that its $N(m) = \sqrt{m}$, $|\mathcal{M}|=O(L)$, $|\mathcal{N}|=O(L^2)$, with $L=O(\log m)$ or constant (i.e., $L =O(1)$). Substituting these into the first condition in Lemma \ref{lemma:conc_mtx_error} gives the desired result.
\end{proof}

\begin{lemma} \label{lemma:realizable}
Assume the conditions for Lemma \ref{lemma:conc_mtx_error} hold. With probability at least $1-\delta$, there exists a parameter vector ${\thetabf}^* \in \R^{p}$ ($p=\mathrm{dim}( {\thetabf})$) such that for all contexts $\xbf^i \in \mathcal{X}_{1:TK}$:
\begin{enumerate}
    \item The reward function is perfectly represented by a linear model:
    $$ h(\xbf^i) = \langle \nabla_{\thetabf} f(\xbf^i;\thetabf_0), {\thetabf}^* - {\thetabf}_0 \rangle. $$
    \item The norm of the solution vector is bounded around the initial parameter $\bm{\theta}_0$:
    $$ N_K(m)\norm{{\thetabf}^* - {\thetabf}_0}_2^2 \le 2 \hbf^\top \bar{\mathbf{K}}^{-1} \hbf. $$
\end{enumerate}
Here $\hbf = [h(\xbf^1), \dots, h(\xbf^{TK})]^\top$ is the vector of true rewards.
\end{lemma}

\begin{proof}
By Lemma \ref{lemma:conc_mtx_error} and a union bound, choosing $\varepsilon'=\lambda_0/(2TK)$ ensures that, for a sufficiently large $m$ and with probability at least $1-\delta$,
\[
\|\hat{\mathbf{K}}-\bar{\mathbf{K}}\|_{\mathrm{F}}  \;\le\; TK\,\varepsilon' \;=\; \frac{\lambda_0}{2}.
\]
Hence, 
\[
\hat{\mathbf{K}} \;\succeq\; \bar{\mathbf{K}} - \|\hat{\mathbf{K}}-\bar{\mathbf{K}}\|_2\,\Ibf
\;\succeq\; \bar{\mathbf{K}} - \|\hat{\mathbf{K}}-\bar{\mathbf{K}}\|_{\mathrm{F}}\,\Ibf
\;\succeq\; \bar{\mathbf{K}} - \frac{\lambda_0}{2}\Ibf
\;\succeq\; \frac{\lambda_0}{2}\Ibf \;\succ\; 0,
\]
the second inequality is by $\|\mathbf{A}\|_2\le\|\mathbf{A}\|_{\mathrm{F}}$ and fourth inequality by $\bar{\mathbf{K}}\succeq \lambda_0 \Ibf$.
Therefore, $\hat{\mathbf{K}}$ is positive definite.

Define $\Jbf=\Jbf_0=\Jbf({\thetabf}_0)\in\R^{p\times TK}$ be the Jacobian at initialization (its columns are the gradient vectors $\nabla_{\thetabf} f(\xbf^i;{\thetabf}_0)$). By definition,
\[
\hat{\mathbf{K}} = \frac{1}{N_K(m)}\,\Jbf^\top \Jbf.
\]
Since $\hat{\mathbf{K}}\succ 0$, the matrix $\Jbf^\top \Jbf$ is invertible and $\Jbf$ has full rank. Define
\[
{\thetabf}^* = {\thetabf}_0+\Jbf\,(\Jbf^\top \Jbf)^{-1} \hbf.
\]
Then
\[
\Jbf^\top({\thetabf}^*-{\thetabf}_0) = \Jbf^\top \Jbf\,(\Jbf^\top \Jbf)^{-1} \hbf \;=\; \hbf,
\]
i.e., for each $i$, $\langle \nabla_{\thetabf} f(\xbf^i;{\thetabf}_0),\,{\thetabf}^*-{\thetabf}_0\rangle = h(\xbf^i)$, proving part~1 of the lemma.

For part 2 of the lemma, note by construction,
\[
\|{\thetabf}^*-{\thetabf}_0\|_2^2
\;=\; \hbf^\top (\Jbf^\top \Jbf)^{-1} \hbf
\;=\; \frac{1}{N_K(m)}\, \hbf^\top \hat{\mathbf{K}}^{-1} \hbf.
\]
Recall $\hat{\mathbf{K}} \succeq \bar{\Kbf}-\frac{\lambda_0}{2} \Ibf \succeq \frac12 \bar{\Kbf}$, so $\hat{\mathbf{K}}^{-1} \preceq 2\,\bar{\mathbf{K}}^{-1}$. Hence
\[
\|{\thetabf}^*-{\thetabf}_0\|_2^2 \;\le\; \frac{1}{N_K(m)}\, \hbf^\top (2\,\bar{\mathbf{K}}^{-1}) \hbf
\;=\; \frac{2}{N_K(m)}\, \hbf^\top \bar{\mathbf{K}}^{-1} \hbf,
\]
\end{proof}

\subsection{Confidence bounds and instantaneous regret}

\begin{lemma}\label{lem:selfnorm}
Fix $\lambda>0$ and $\delta\in(0,1)$. We use notations from Algorithm \ref{alg:qntk-ucb}. 
With probability at least $1-\delta$, for all $t$,
\begin{equation}
\label{eq:conf}
\norm{\sqrt{N_K(m)}({\thetabf}^* - {\thetabf}_0)-(\hat{\thetabf}_t-\thetabf_0)}_{\Zbf_t}  \leq \beta_t
\quad\mbox{where}\quad
\beta_{t} = \nu\sqrt{\log\frac{\det(\Zbf_{t})}{\det(\lambda \Ibf)} + 2\log\Big(\frac{1}{\delta}\Big)}
+\sqrt{\lambda}S,
\end{equation}
where $\Zbf_t = \lambda \Ibf + \sum^t_{s=1} \phibf(\xbf_{s,a_s})\phibf(\xbf_{s,a_s})^\top$ and $S\geq \sqrt{ 2\hbf^\top \bar{\Kbf}^{-1} \hbf  }$. 
\end{lemma}

\begin{proof}
Since we have proved in Lemma \ref{lemma:realizable} that the reward function is perfectly represented by a linear model, i.e., 
    $$ h(\xbf^i) = \langle \nabla_{\thetabf} f(\xbf^i;\thetabf_0)/\sqrt{N_K(m)}, \sqrt{N_K(m)}({\thetabf}^* - {\thetabf}_0) \rangle = \langle \phibf(\xbf^i), \sqrt{N_K(m)}({\thetabf}^* - {\thetabf}_0) \rangle, $$
 a direct application of the self-normalized martingale inequality for linear bandits
(cf. \cite[Theorem 2]{linucb2011}) yields that $$
\norm{\sqrt{N_K(m)}({\thetabf}^* - {\thetabf}_0)-\Zbf^{-1}_t\bbf_t}_{\Zbf_t}
=
\norm{\sqrt{N_K(m)}({\thetabf}^* - {\thetabf}_0)-(\hat{\thetabf}_t-\thetabf_0)}_{\Zbf_t}  \leq \beta_t
$$
which is the bound in~\eqref{eq:conf}.
Furthermore, the norm of the solution vector ${\thetabf}^*$ is bounded. More precisely, 
$$ N_K(m)\norm{{\thetabf}^* - {\thetabf}_0}_2^2 \le 2\, \hbf^\top \bar{\mathbf{K}}^{-1} \hbf, $$
which justifies the condition on $S$ in the lemma statement.
\end{proof}

\begin{lemma}\label{lem:inst}
Suppose the
confidence event in~\eqref{eq:conf} holds at round $t$, i.e.,
\[
\norm{\sqrt{N_K(m)}({\thetabf}^* - {\thetabf}_0)-(\hat{\thetabf}_{t-1}-\thetabf_0)}_{\Zbf_{t-1}}  \leq \beta_{t-1},
\]
furthermore, assume that $\lambda \geq \max\{1,S^{-2}\}$. Then define $\widetilde{r}_t=h(\xbf_{t,a_t^*})-h(\xbf_{t,a_t})$,
\[
\widetilde{r}_t \le 2\,\beta_{t-1}\min\bigl\{\|\phibf(\xbf_{t,a_t})\|_{\Zbf_{t-1}^{-1}},\,1\bigr\}.
\]
\end{lemma}

\begin{proof}
Define for any context $\xbf$ the (time $t$) predicted (posterior) mean and variance
\[
\hat{\mu}_{t}(\xbf)\;:=\;\phibf(\xbf)^\top (\hat{\bm{\theta}}_{t-1} - \thetabf_0),
\qquad
s_{t}(\xbf)\;:=\;\|\phibf(\xbf)\|_{\Zbf_{t-1}^{-1}},
\]
and the corresponding optimistic and pessimistic indices
\[
U_{t}(\xbf)\;:=\;\hat{\mu}_{t}(\xbf)+\beta_{t-1}s_{t}(\xbf),
\qquad
L_{t}(\xbf)\;:=\;\hat{\mu}_{t}(\xbf)-\beta_{t-1}s_{t}(\xbf).
\]
On the event \eqref{eq:conf}, the Cauchy–Schwarz inequality in the
$\Zbf_{t-1}$-norm gives, for every $\xbf$,
\begin{align*}
\bigl| h(\xbf)-\hat{\mu}_{t}(\xbf)\bigr|
&=\bigl|\phibf(\xbf)^\top(\sqrt{N_K(m)}(\bm{\theta}^*- \thetabf_0)-(\hat{\bm{\theta}}_{t-1}-\thetabf_0))\bigr| \\
&\le \|\phibf(\xbf)\|_{\Zbf_{t-1}^{-1}}\;\|\sqrt{N_K(m)}(\bm{\theta}^*- \thetabf_0)-(\hat{\bm{\theta}}_{t-1}-\thetabf_0)\|_{\Zbf_{t-1}} \\
&\le \beta_{t-1}s_{t}(\xbf),
\end{align*}
hence
\begin{equation}\label{eq:ci-band}
L_{t}(\xbf)\;\le\; h(\xbf)\;\le\; U_{t}(\xbf)\qquad \text{for all $\xbf$.}
\end{equation}
Let $a_t\in\arg\max_{a\in[K]}U_{t}(\xbf_{t,a})$ be the action chosen by the algorithm, and let
$a_t^*\in\arg\max_{a\in[K]}\,h(\xbf_{t,a})$ be an optimal action. By the optimism principle and \eqref{eq:ci-band},
\[
h(\xbf_{t,a_t^*})
\;\le\; U_{t}(\xbf_{t,a_t^*})
\;\le\; U_{t}(\xbf_{t,a_t}),
\qquad
L_{t}(\xbf_{t,a_t})
\;\le\; h(\xbf_{t,a_t}).
\]
Subtracting the rightmost inequality from the leftmost inequality yields
\[
\widetilde{r}_t
= h(\xbf_{t,a_t^*})-h(\xbf_{t,a_t})
\;\le\; U_{t}(\xbf_{t,a_t}) - L_{t}(\xbf_{t,a_t})
= 2\,\beta_{t-1}\,s_{t}(\xbf_{t,a_t})
= 2\,\beta_{t-1}\,\|\phibf(\xbf_{t,a_t})\|_{\Zbf_{t-1}^{-1}}.
\]
This proves  that $\widetilde{r}_t\le 2\beta_{t-1}\|\phibf(\xbf_{t,a_t})\|_{\Zbf_{t-1}^{-1}}$.

Also, since rewards are bounded in $[0,1]$, we also have $\widetilde{r}_t\le 1$. Combining the two bounds gives
\[
\widetilde{r}_t \;\le\; \min\bigl\{\,2\beta_{t-1}\|\phibf(\xbf_{t,a_t})\|_{\Zbf_{t-1}^{-1}},1\bigr\}.
\]
Finally, we obtain the stated result since $\beta_{t-1} \geq \sqrt{\lambda}S\geq 1$.
\end{proof}

The following is similar to  the elliptical potential lemma~\cite{linucb2011}.
\begin{lemma}
\label{lem:elliptical}
The following holds. 
\[
\sum_{t=1}^T \min\!\left\{ \norm{\phibf(\xbf_{t,a_t})}_{\Zbf_{t-1}^{-1}}^2,\,1 \right\}
\;\le\; 2\,\log\frac{\det(\Zbf_T)}{\det(\lambda \Ibf)}.
\]
\end{lemma}

\begin{proof}
We follow the proof of   \cite[Lemma 11]{linucb2011}.

First, elementary matrix identities give,
\begin{align*}
    \det(\Zbf_t) &=\det(\Zbf_{t-1}+\phibf(\xbf_{t,a_t})\phibf(\xbf_{t,a_t})^\top) \\
&=\det(\Zbf_{t-1}) \bigl(1+\phibf(\xbf_{t,a_t})^\top \Zbf_{t-1}^{-1}\phibf(\xbf_{t,a_t}) \bigr)\\
&=\det(\Zbf_{t-1}) \bigl(1+\|\phibf(\xbf_{t,a_t})\|_{\Zbf_{t-1}^{-1}}^2\bigr) \\
&= \det(\lambda \Ibf) \prod^t_{s=1} \bigl(1+\|\phibf(\xbf_{s,a_s})\|_{\Zbf_{s-1}^{-1}}^2\bigr)
\end{align*}
Telescoping over $t=1,\dots,T$ and taking logs gives
\begin{equation}\label{eq:logdet-sum}
\log\frac{\det(\Zbf_T)}{\det(\lambda \Ibf)}
=\sum_{t=1}^T \log\bigl(1+\|\phibf(\xbf_{t,a_t})\|_{\Zbf_{t-1}^{-1}}^2\bigr).
\end{equation}
Note for all $x \in [0,1]$, $x\leq 2\log(1+x)$. Hence we have $\min\{x,1\}\ \le\ 2\log(1+x)$.
Hence 
\[
\sum_{t=1}^T \min\bigl\{\|\phibf(\xbf_{t,a_t})\|_{\Zbf_{t-1}^{-1}}^2,1\bigr\}
\ \le\ 2\sum_{t=1}^T \log\bigl(1+\|\phibf(\xbf_{t,a_t})\|_{\Zbf_{t-1}^{-1}}^2\bigr)
=
2\log\frac{\det(\Zbf_T)}{\det(\lambda \Ibf)}.
\]
This proves the lemma.
\end{proof}

\begin{lemma}\label{lem:logdet-qntk}
Let $\hat{\mathbf K}\in\mathbb{R}^{TK\times TK}$ be the empirical QNTK Gram matrix over all contexts
$\mathcal{X}_{1:TK}=\{\xbf_{t,a}\}_{t\in[T],a\in[K]}$, and let $\bar{\mathbf K}$ be the limiting QNTK Gram matrix  on the same set.
Define the spectral mismatch
$\bm{\Delta} := 
\hat{\mathbf K}-\bar{\mathbf K}
$.
Then
\begin{equation}\label{eq:logdet-main}
\log\frac{\det(\Zbf_T)}{\det(\lambda \Ibf)}
\;\le\; \log\det\Big(\Ibf+\frac{1}{\lambda}\,\bar{\mathbf K}\Big)
\;+\; \frac{\sqrt{TK}}{\lambda}\,\|\bm{\Delta}\|_F.
\end{equation}
Consequently, invoking the definition of the quantum effective dimension
\[
\widetilde{d}_{\mathrm{q}}:= \frac{\log\det\!\big(\Ibf+\bar{\mathbf K}/\lambda\big)}{\log\!\big(1+TK/\lambda\big)},
\]
we have 
\[
\log\frac{\det(\Zbf_T)}{\det(\lambda \Ibf)}
\;\le\; \widetilde{d}_{\mathrm{q}}\,\log\!\bigg(1+\frac{TK}{\lambda}\bigg)
\;+\; \frac{\sqrt{TK}}{\lambda}\,\|\bm{\Delta}\|_F.
\]
\end{lemma}

\begin{proof}
Consider,
\begin{align}
\log\frac{\det( \Zbf_T)}{\det(\lambda \Ibf)}
&= \log\det\!\Big(\Ibf + \frac{1}{\lambda}\sum_{t=1}^T  \phibf(\xbf_{t,a_t})\,\phibf(\xbf_{t,a_t})^\top\Big)\nonumber\\
&\le \log\det\!\Big(\Ibf + \frac{1}{\lambda}\sum_{i=1}^{TK} \phibf(\xbf^i) \phibf(\xbf^i)^\top\Big)\nonumber\\
&= \log\det\!\Big(\Ibf + \frac{1}{\lambda\, N_K(m)} \Jbf\Jbf^\top\Big) \\
&= \log\det\!\Big(\Ibf + \frac{1}{\lambda\, N_K(m)} \Jbf^\top \Jbf\Big). \label{eq:swap}
\end{align}

Now $\frac{1}{N_K(m)}\Jbf^\top \Jbf$ is exactly $\hat{\Kbf}$, writing
$\hat{\Kbf}=\bar{\Kbf}+\bm{\Delta}$, we have 
\begin{align*}
\log\det\!\Big(\Ibf+\frac{1}{\lambda}\hat{\Kbf}\Big)
&= \log\det\!\Big(\Ibf+\frac{1}{\lambda}(\bar{\Kbf}+\bm{\Delta})\Big)\\
&\le \log\det\!\Big(\Ibf+\frac{1}{\lambda}\bar{\Kbf}\Big)
\;+\;\Big\langle \frac{1}{\lambda}\Big(\Ibf+\frac{1}{\lambda}\bar{\Kbf}\Big)^{-1},\,\bm{\Delta}\Big\rangle \\
&\le \log\det\!\Big(\Ibf+\frac{1}{\lambda}\bar{\Kbf}\Big)
\;+\;\frac{1}{\lambda}\,\big\|\big(\Ibf+\frac{1}{\lambda}\bar{\Kbf}\big)^{-1}\big\|_{\mathrm{F}}\,\|\bm{\Delta}\|_{\mathrm{F}}  \\
&\le \log\det\!\Big(\Ibf+\frac{1}{\lambda}\bar{\Kbf}\Big)
\;+\;\frac{\sqrt{TK}}{\lambda}\,\|\bm{\Delta}\|_{\mathrm{F}},
\end{align*}
where we used the concavity of $\log\det (\cdot)$, $|\langle \mathbf A,\mathbf B\rangle|\le \|\mathbf A\|_{\mathrm{F}}\,\| \mathbf B\|_{\mathrm{F}}$ and
$\|\mathbf A\|_{\mathrm{F}} \leq \sqrt{TK}\|\mathbf A\|_2$ for any $\mathbf A \in \mathbb R^{TK \times TK}$.

So we have 
\begin{align*}
\log\frac{\det(\Zbf_T)}{\det(\lambda \Ibf)}
&\le \log\det\!\Big(\Ibf+\frac{1}{\lambda}\bar{\Kbf}\Big)
\;+\;\frac{\sqrt{TK}}{\lambda}\,\|\bm{\Delta}\|_{\mathrm{F}} \\
&\leq \widetilde{d}_{\mathrm{q}}\,\log\!\Big(1+\frac{TK}{\lambda}\Big)
\;+\; \frac{\sqrt{TK}}{\lambda}\,\|\bm{\Delta}\|_{\mathrm{F}}.
\end{align*}

\end{proof}


\subsection{Proof of Theorem \ref{thm:regret}}

\thmregret*

\begin{proof}
Consider,
\begin{align*}
    R_T &= \sum_{t=1}^T\big(h(\xbf_{t,a_t^*})-h(\xbf_{t,a_t})\big) \\
    &\leq \sum_{t=1}^T 2\beta_{t-1}\min\bigl\{\,\|\phibf(\xbf_{t,a_t})\|_{\Zbf_{t-1}^{-1}},\;1\,\bigr\} \\
    &\leq 2\beta_T \sum_{t=1}^T \min\bigl\{\,\|\phibf(\xbf_{t,a_t})\|_{\Zbf_{t-1}^{-1}},\;1\,\bigr\} \\
    &\leq 2\beta_T \sqrt{T} \sqrt{\sum_{t=1}^T \min\bigl\{\,\|\phibf(\xbf_{t,a_t})\|^2_{\Zbf_{t-1}^{-1}},\;1\,\bigr\}} \\
    &\leq 2\,\beta_T\,\sqrt{\,2T\;\log\frac{\det(\Zbf_T)}{\det(\lambda \Ibf)}\,},
\end{align*}
where the third inequality is by the Cauchy--Schwarz inequality. Recall that the spectral mismatch $\bm{\Delta} := 
\hat{\mathbf K}-\bar{\mathbf K}$.
Substituting  
\[ 
\log\frac{\det( \Zbf_T)}{\det(\lambda \Ibf)}
\;\le\; \widetilde{d}_{\mathrm{q}}\,\log\!\Big(1+\frac{TK}{\lambda}\Big) + \frac{\sqrt{TK}}{\lambda}\|\bm{\Delta}\|_{\mathrm{F}}
\]
and $$\beta_{T} \;=\; \nu\sqrt{\log\frac{\det(\Zbf_{T})}{\det(\lambda \Ibf)}+2\log\Big(\frac{1}{\delta}\Big)}
+\sqrt{\lambda}S$$ gives the bound
\begin{align*}
    R_T&\le
2 \sqrt{2T\left( \widetilde{d}_{\mathrm{q}} \log\Big(1+\frac{TK}{\lambda}\Big)+\frac{\sqrt{TK}}{\lambda}\|\bm{\Delta}\|_{\mathrm{F}} \right)} \\
&\qquad \times\Big(\nu\sqrt{\left( \widetilde{d}_{\mathrm{q}}\log\Big(1+\frac{TK}{\lambda}\Big)+\frac{\sqrt{TK}}{\lambda}\|\bm{\Delta}\|_{\mathrm{F}} \right) + 2\log\Big(\frac{1}{\delta}\Big)}+\sqrt{\lambda}S\Big).
\end{align*}
By Lemma \ref{lemma:conc_mtx_error} and \ref{lemma:qnn_structure}, $m = \bigOmega\left( \frac{T^3K^3}{\lambda^2}\log\left(\frac{(TK)^2}{\delta}\right) \right)$ gives $\|\bm{\Delta}\|_{\mathrm{F}}=
\big\|\hat{\mathbf K}-\bar{\mathbf K}\big\|_{\mathrm{F}}  \;\le\; TK\varepsilon \leq \frac{\lambda}{\sqrt{TK}}
$. This yields  the stated bound.
\end{proof}

\end{appendices}


\end{document}